\documentclass[twoside,11pt]{article}

\usepackage[abbrvbib]{jmlr2e}
\usepackage{times}
\usepackage{microtype}
\usepackage{bm}

\usepackage{macrosArticle}

\let\norm\undefined 
\let\set\undefined 
\usepackage{notation}

\let\pr\bP
\let\ex\bE
\let\top\intercal
\DeclareBoldMathCommand{\vpi}{\pi}
\DeclareBoldMathCommand{\vmu}{\mu}
\DeclareBoldMathCommand{\vlambda}{\lambda}
\DeclareBoldMathCommand{\c}{c}
\DeclareBoldMathCommand{\w}{w}
\DeclareBoldMathCommand{\v}{v}
\DeclareMathOperator{\Alt}{Alt}
\def\UCB{\mathrm{U}}
\def\LCB{\mathrm{L}}
\def\KL{\mathrm{KL}}
\def\ihat{\hat{\imath}}

\let\inf\undefined
\DeclareMathOperator*{\inf}{\vphantom{p}inf}

\usepackage{todonotes}

\usepackage{subfig,float}

\let\log\ln

\usepackage{lastpage}
\jmlrheading{22}{2021}{1-\pageref{LastPage}}{11/18; Revised
9/21}{11/21}{18-798}{Emilie Kaufmann and Wouter M. Koolen}
\ShortHeadings{Mixture Martingales Revisited}{Kaufmann and Koolen}
\firstpageno{1}

\begin{document}

\title{Mixture Martingales Revisited \\with Applications to Sequential Tests and Confidence Intervals}

\author{\name{Emilie Kaufmann}\\
  \addr{Univ. Lille, CNRS, Inria,  Centrale Lille, UMR 9189 - CRIStAL, \\ F-59000 Lille, France}
  \email{emilie.kaufmann@univ-lille.fr}
  \AND
  \name{Wouter M. Koolen} \\
  \addr{Centrum Wiskunde \& Informatica, Science Park 123, Amsterdam, Netherlands}
  \email{wmkoolen@cwi.nl}
}

\editor{Csaba Szepesv\'ari}

\maketitle

\begin{abstract}
  This paper presents new deviation inequalities that are valid uniformly in time under adaptive sampling in a multi-armed bandit model. The deviations are measured using the Kullback-Leibler divergence in a given one-dimensional exponential family, and take into account \emph{multiple arms} at a time. They are obtained by constructing for each arm a mixture martingale based on a hierarchical prior, and by multiplying those martingales. Our deviation inequalities allow us to analyze stopping rules based on generalized likelihood ratios for a large class of sequential identification problems. We establish asymptotic optimality of sequential tests generalising the track-and-stop method to problems beyond best arm identification. We further derive sharper stopping thresholds, where the number of arms is replaced by the newly introduced pure exploration problem rank. We construct tight confidence intervals for linear functions and minima/maxima of the vector of arm means.
\end{abstract}


\section{Introduction}\label{sec:intro}

We are interested in making decisions under uncertainty in its myriad forms, including sequential allocation and hypothesis testing problems. In this paper our goal is the design of tight confidence regions that are valid uniformly in time, as well as the design of efficient stopping rules for a large class of sequential tests.

We will develop our results in the standard multi-armed bandit model with $K$ independent one-dimensional exponential family \emph{arms} that are parameterised by their means $\vmu = (\mu_1, \ldots, \mu_K)$. In this setup, samples $X_1,X_2\dots$ are sequentially gathered from the different arms: $X_t$ is drawn from the distribution that has mean $\mu_{A_t}$ where $A_t \in \{1,\dots,K\}$ is the arm selected at round $t$. Our techniques all make use of \emph{self-normalised sums}, which are defined after $t$ rounds by
\begin{equation}\label{eq:snss}
  \sum_{a \in \cS} N_a(t) d(\hat \mu_a(t), \mu_a).
\end{equation}
Here $\cS$ is a subset of the arms $\{1,\dots,K\}$, $N_a(t)$ is the \emph{random} number of observations from arm $a$, $\hat \mu_a(t)$ is the empirical mean of these observations after $t$ rounds, and $d(\mu, \lambda) \ge 0$ is the relative entropy (Kullback-Leibler divergence) from the exponential family distribution with mean $\mu$ to that with mean $\lambda$. The more the empirical means of arms in $\cS$ deviate from the true means, the larger the self-normalised sum. We call the summands self-normalised as they are KL-based analogues of the (squared) $t$-statistic. Namely, a second-order Taylor expansion in $\mu$ around $\hat \mu(t)$ reveals that $N(t) d(\hat \mu(t), \mu) \approx N(t) \frac{(\hat \mu(t)- \mu)^2}{2 \mathbb V(\hat \mu(t))}$, where $\mathbb V(\mu)$ is the variance of the model with mean $\mu$. One of the reasons why self-normalized sums show up in different sequential learning problems is their relation to (generalized) log likelihood ratio statistics. For example, it can be shown that \[\ln \frac{\ell(X_1,\ldots,X_t ; \hat{\vmu}(t))}{\ell(X_1,\ldots,X_t ; \bm{\mu})} = \sum_{a=1}^{K}N_a(t) d (\hat{\mu}_a(t),\mu_a)\] where $\ell(X_1,\ldots,X_t ; \bm \lambda)$ is the likelihood of the observations under a bandit model whose vector of means is $\bm \lambda$ and $\hat{\vmu}(t) = (\hat{\mu}_1(t),\dots,\hat{\mu}_K(t))$.

The proposed analyses of the sequential procedures discussed in this paper all rely on a tight control of the deviations of self-normalized sums of the form \eqref{eq:snss}, which inform us about possible values of the means. Our first contribution is the construction of explicit \emph{calibration functions} $\cC(x) = x + o(x)$ for which, under any sampling rule (effecting the $N_a(t)$ sampling counts), any bandit model $\vmu$ and any confidence $\delta \in (0,1)$, the self-normalised sum associated to \emph{any subset of arms} $\cS$ satisfies
\begin{equation}\label{eq:stylised}
  \pr_\vmu \left(
    \exists t \in \N :
    \sum_{a \in \cS} \sbr[\Big]{
      N_a(t) d(\hat \mu_a(t), \mu_a)
      - O(\ln \ln N_a(t))
    }
    \ge
    |\cS| \cC\del*{\frac{\ln \frac{1}{\delta}}{|\cS|}}
  \right)
  ~\le~
  \delta
  .
\end{equation}
The salient features of this result are that it is uniform in time, exploits the information geometry (KL) intrinsic to the exponential family (rather than relying on non-parametric relaxations including sub-Gaussianity), and, more importantly, it generalises confidence ellipses by combining in the strong summation sense the evidence from \emph{multiple arms}. Furthermore, as we develop inequalities that hold for any subset $\cS$, at the moderate price of a weighted union bound we may apply the bound to any arbitrary (random) subset of the arms, and thereby control the model-selection trade off between the amount of evidence on the left and the magnitude of the threshold on the right.

We may recognise two well-known statistical effects (i.e.\ fundamental barriers) in the form of the bound \eqref{eq:stylised}. First, the Law of the Iterated Logarithm informs us that, in the Gaussian case, $\lim\sup_{N_a(t) \to \infty} \frac{N_a(t) d(\hat \mu_a(t), \mu_a)}{\ln \ln N_a(t)} = \lim\sup_{N_a(t) \to \infty} \frac{N_a(t) (\hat \mu_a(t)- \mu_a)^2}{\ln \ln N_a(t)} $ is a universal constant a.s., whence the correction in the sum. Moreover, it follows from the Wilks phenomenon \citep{Wilks38}, which gives the limit distribution of Generalized Likelihood Ratio statistics, that, when $\cS = \{1,\dots,K\}$, twice the self-normalised sum \eqref{eq:snss} converges in distribution to a $\chi^2_{K}$ distribution. The $K$ degrees of freedom are reflected in the perspective scaling of the threshold to which $\sum_{a=1}^{K}N_a(t) d (\hat{\mu}_a(t),\mu_a)$ is compared in \eqref{eq:stylised}.

%

The formal statement of our concentration inequalities is given in Section~\ref{sec:expfam.case}, in which we prove a general result that holds for any exponential family (Theorem~\ref{thm:DevExpo}) and state two improved results for Gaussian and Gamma distributions (Theorems~\ref{corr:Gaussian} and~\ref{corr:Gamma} respectively). We now compare our results to previous work and explain why measuring deviations over multiple arms simultaneously is crucial for applications to sequential learning, which we discuss in Sections~\ref{sec:generic.solved} to~\ref{sec:ProjectedCI}.

\subsection{Novelty of our Concentration Results}

Due to the sequential nature of the data collection process, the analysis of virtually any bandit algorithm relies on deviation inequalities that can take into account the random number of observations from each arm. Several such inequalities have thus been developed in this literature and beyond. 

However, most of these results measure deviations for one arm only, which can be rephrased in the form of the following time-uniform deviation inequality
\begin{equation}\label{eq:stylised2}
  \pr_\mu \left(
    \exists t \in \N :
    t d(\hat \mu_t, \mu)
      - O(f(t))
    \ge
     \cC\del*{\ln \frac{1}{\delta}}
  \right)
  ~\le~
  \delta,
\end{equation}
where $\hat{\mu}_t$ is the empirical average of $t$ i.i.d.\ observations with mean $\mu$ in a one-parameter exponential family and $f(t) = \ln(t)$ or $\ln\ln(t)$.\footnote{Some existing results rather upper bound the probability that $|\hat{\mu}_t - \mu|$ exceeds some threshold. For one-parameter exponential families, we think that it is more natural to measure deviations with the KL-divergence function as the Cramér-Chernoff inequality for such distributions can be expressed as $\bP(td(\hat{\mu}_t,\mu) > \ln(1/\delta), \hat{\mu}_t > \mu) \leq \delta$.
  This form is also more convenient for measuring deviations for multiple arms, which is supported by our new inequalities.}
Further, the majority of existing inequalities were obtained for Gaussian (or sub-Gaussian) distributions with thresholds featuring $f(t)=\ln(t)$ \cite[e.g.][]{DeLaPenaal04,Maillard19HDR} or $f(t) = \ln\ln(t)$ \citep{Robbins70LIL,Jamiesonal14LILUCB,JMLR15,ZhaoZSE16,Howard18Bernstein}.
For other one-dimensional exponential families, time-uniform deviation inequalities with $f(t) = \log(t)$ have been stated for Bernoulli \citep{Lai76,Jonsson20} and Gamma distributions \citep{Lai76}. \cite{Lai76} also provides a generic recipe for general one-parameter exponential families, but that leads to intractable thresholds. On the contrary, our Theorem~\ref{thm:DevExpo} applied to $|\cS|=1$ leads to an explicit inequality of the form \eqref{eq:stylised2} for any exponential family, with a scaling in $f(t) = \log\log(t)$. The closest existing result is that of \cite{AOKLUCB}, which controls the deviations uniformly for $t$ in a finite time range $\{1,\dots,n\}$.

To the best of our knowledge, the only prior result that controls deviations over multiple arms simultaneously is Theorem 2 of \cite{Combes14Lip}, which also bounds deviations for $t$ in a finite time range $\{1,\dots,n\}$. We provide a detailed comparison with this result in Section~\ref{sec:expfam.case}, showing that our Theorem~\ref{thm:DevExpo} leads to tighter thresholds, which are furthermore valid for the entire time range $t\in \N$. Given the large number of results that are available for $|\cS|=1$, a natural question is whether inequalities like \eqref{eq:stylised2} for different arms can be combined to obtain an inequality like \eqref{eq:stylised}. There is no straightforward way to do so and obtain the right scaling in $\delta$: using a naive union bound leads to an inequality of the form \eqref{eq:stylised} in which the right-hand side is $|\cS|\cC(\ln(1/\delta)) \simeq |\cS|\ln(1/\delta)$ instead of $|\cS|\cC(\ln(1/\delta)/|\cS|) \simeq \ln(1/\delta)$. Hence, specific techniques are needed to propose deviation inequalities that sum evidence across arms, which we provide.

In this work we obtain essentially tight calibration functions by building suitable martingales. We show that a calibration function $\cC$ satisfying \eqref{eq:stylised} can be obtained by exhibiting a martingale that multiplicatively dominates $\exp\left(\lambda\left[N_a(t) d(\hat \mu_a(t), \mu_a)- O(\ln \ln N_a(t))\right]\right)$ for a suitable $\lambda \in (0,1)$. This central assumption to derive deviation inequalities that sum evidence across arms is formalized in Section~\ref{sec:general.deviations}. Our results are then obtained by leveraging some particular martingales called \emph{mixture martingales} that have this property, which are defined in Section~\ref{subsec:MixtureMartingales}.

Using martingales to obtain time-uniform inequalities is an old idea that can be traced back to \cite{Ville39} and all the concentration results quoted above also rely on martingales. We refer the reader to the recent survey of \citet{Howard20Survey} who study in great detail the power of elementary martingales for deriving time-uniform inequalities, yet without the particular focus on exponential families or multiple arms that we adopt here. 
Two important techniques based on martingales are the use of a peeling trick (see, e.g.\ \citealt{KLUCBJournal}) or the ``method of mixtures'' that has been popularized by \cite{DeLaPenaal04,DeLaPenaal09Book}, and is sometimes also referred to as the Laplace method \citep{Maillard19HDR}. We refer the reader to the discussion in Section~\ref{subsec:MixtureMartingales} for examples of use of mixture martingales. Our results rely on new constructions of mixture martingales that are tailored for exponential families. Interestingly, we note that the result of \cite{Combes14Lip} is not based on mixture martingales: its proof relies on a peeling technique which requires the knowledge of $n$, and a stochastic dominance argument. Our proof technique based on mixture martingales is more flexible as it allows to easily bound deviations uniformly over the entire domain $t\in \N$, which is crucial for the analysis of sequential tests that involve random stopping.

\subsection{Applications to Sequential Learning}
In this section we give more context, review our contribution, and illustrate its advantage on a simple example.

\subsubsection{Related Work on Bandits} Stochastic multi-armed bandit models can be traced back to the work of \cite{Thompson33} motivated by clinical trials. They were later studied by \cite{Robbins52Freq,LaiRobbins85bandits} who introduced the regret minimization objective: the samples $X_1,\dots,X_t$ are seen as rewards and the goal is to find a sequential strategy to maximize the (expected) cumulated reward, which is equivalent to minimizing some notion of regret \cite[see e.g.][for surveys]{Bubeck:Survey12,BanditBook}.

In the meantime, pure-exploration problems in bandit models have also received increased attention \citep{EvenDaral06,Bubeckal11}. In this context, a common objective is to identify as quickly and accurately as possible the arm with the largest mean, relinquishing the incentive to maximize the sum of rewards. In the fixed-confidence setting, the minimal number of samples needed to identify the best arm with accuracy larger than $1-\delta$ when arms belong to a one-dimensional family has been identified by \cite{GK16}, in a regime of small values of $\delta$. Their Track-and-Stop algorithm is shown to asymptotically match this optimal sample complexity. Extensions of this best arm identification problem in which one should answer quickly and accurately some more general query about the means of the arms have also been studied \citep{HuangASM17,ChenGLQW17}. Prototypical queries beyond Best Arm include Top-$M$ \citep{Shivaram:al10}, Thresholding \citep{thresholding}, Minimum Threshold \citep{kaufmann2018sequential}, Combinatorial Bandits \citep{Chen14ComBAI}, pure-strategy Nash equilibria \citep{pmlr-v70-zhou17b} or Monte-Carlo Tree Search \citep{Teraoka14MCTS}. 
We note that Track-and-Stop has recently been generalized by \cite{Juneja19} to a generic ``partition identification'' problem similar to the one that we consider in Section~\ref{sec:generic.solved}, while \cite{multiple.answers} have studied its extension to queries with multiple correct answers.
Finally, recent research has also focused on developing alternatives to Track-and-Stop that are more efficient numerically, like \cite{purex.games} who develop algorithms based on iterative saddle point solving.

\subsubsection{Our Contributions} The first impact of our concentration results is that they permit to analyse new stopping rules based on Generalized Likelihood Ratios, which extend the stopping rule originally proposed for Track-and-Stop \citep{GK16} to generic sequential identification problems. Our generic stopping rule is presented in Section~\ref{sec:generic.solved}, in which we further show that under some assumptions on the identification problem itself, such a stopping rule combined with a suitable sampling rule is (asymptotically) optimal in terms of sample complexity. We then provide in Section~\ref{sec:stop.for.test} refined stopping criteria for some particular tests that replace the number of arms $K$ in the threshold by a new notion of rank.

Next, we explain in Section~\ref{sec:ProjectedCI} how our deviation inequalities can be used to build tight confidence regions on (functions of) the unknown parameter $\bm\mu$. Indeed, the sum form of the left-hand quantity in \eqref{eq:stylised} allows us to build confidence regions that exclude the configuration of all (many) empirical estimates $\hat \mu_a(t)$ being far from their means $\mu_a$ simultaneously. We show how this effect yields improved confidence intervals for functions of the mean $\vmu$ in the cases of linear functions and minima.
In concrete examples, we can quantify the benefit precisely.

\subsubsection{Illustration of the Benefit of \eqref{eq:stylised} on a Simple Example} A common task in sequential learning is to construct a confidence interval on the difference $\mu_1-\mu_2$ in mean between two arms, for example to decide whether $\mu_1$ can plausibly be higher than $\mu_2$ in a best arm identification scenario. We now quantify the benefit of using the self-normalized sum \eqref{eq:stylised} compared to the classical approach of combining per-arm intervals using the union bound, with an illustration provided in Figure~\ref{fig:stylised}.

For maximum interpretability, we instantiate \eqref{eq:stylised} for Gaussian arms with variance 1 (so that $d(x,y) = (x-y)^2/2$), we ignore the $\ln\ln$ terms, and we approximate $K \mathcal C(\ln \tfrac{1}{\delta}/{K}) \approx \ln \tfrac{1}{\delta}$. Then if we follow the classical per-arm approach, we obtain a confidence interval on $\mu_a$ for each arm $a$ separately using \eqref{eq:stylised} (which now reduces to the standard Chernoff bound), combine these into a rectangular confidence region on the pair $(\mu_1, \mu_2)$ using the union bound over arms (called ``Box'' in Figure~\ref{fig:stylised}), and work out what we know about the difference $\mu_1-\mu_2$ by projecting. Doing so, we obtain a confidence interval on $\mu_1 - \mu_2$ that has diameter $\sqrt{8  \ln \frac{2}{\delta}} \del*{\sqrt{\frac{1}{N_a(t)}} + \sqrt\frac{1}{N_b(t)}}$. In contrast, the self-normalised sum of $2$ arms directly provides a confidence ellipse on the pair $(\mu_1, \mu_2)$ (called ``Sum'' in Figure~\ref{fig:stylised}), and projecting that to the difference $\mu_1-\mu_2$ yields a tighter interval of diameter $
   \sqrt{8  \ln \frac{1}{\delta} \left(\frac{1}{N_a(t)}+\frac{1}{N_b(t)}\right)}
   $. The advantage of the second approach can be up to a factor $\sqrt{2}$, which occurs for equal sample sizes $N_a(t)=N_b(t)$. In typical adaptive stopping problems, a reduction by $\sqrt{2}$ in confidence width leads to an improvement by a factor $2$ of the sample complexity. 
   
   In Section~\ref{sec:projected.linear}, we quantify the obtained improvement for the more general task of building a confidence interval on a linear function $\v^\top \vmu$ of the means $\vmu \in \mathbb R^K$, which can be as large as $\sqrt{K}$.


\begin{figure}[h]
  \centering
  \subfloat[Confidence region for $\vmu$]{
    \includegraphics[width=.5\textwidth]{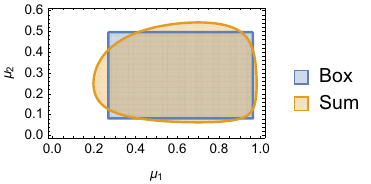}
  }
  \subfloat[Confidence interval for the difference $\mu_1-\mu_2$ obtained by projecting confidence regions for $\vmu$. The dashed grey help lines connect points of equal difference $\mu_1-\mu_2$. The largest and smallest values for the difference are obtained by squeezing the confidence interval between diagonal tangents (solid lines). We see that the confidence width, which is the distance between the intercepts, is strictly larger for Box than for Sum: the rounded nature of Sum provides tighter control on the difference.
  ]{
    \includegraphics[width=.5\textwidth]{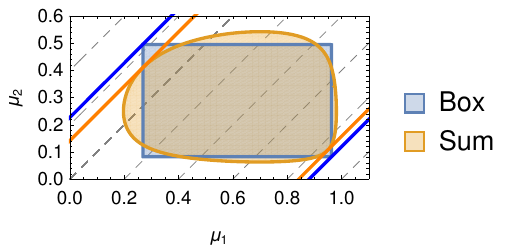}
  }
  \caption{%
    Visual two-arm comparison of confidence regions for $\vmu$ and the implied confidence interval for the difference $\mu_1-\mu_2$. A union bound over traditional per-arm confidence intervals gives the ``Box'' region. Our new bound \eqref{eq:stylised} results in a confidence region of the egg-shape marked ``Sum''.
}\label{fig:stylised}
\end{figure}




\section{Martingales and Deviation Inequalities for Exponential Family Bandit Models} \label{sec:general.deviations}

In this section, we formally introduce the stochastic processes for which we want to obtain deviation inequalities. We then present a general method for obtaining deviation inequalities for any such stochastic process. It relies on the crucial assumption that one can find martingales multiplicatively dominating exponential transforms of the process. We further introduce the general class of martingales that we shall exhibit in order to obtain the particular deviation results of this paper, namely mixture martingales.

\subsection{Exponential Family Bandit Models}

A one-parameter canonical exponential family is a class $\cP$ of probability distributions characterized by a set $\Theta \subset \R$ of natural parameters, a strictly convex and twice-differentiable function $b : \Theta \rightarrow \R$ (called the log-partition function) and a reference measure $m$. It is defined as
\[\cP = \left\{ \nu_{\theta} , \theta \in \Theta : \text{$\nu_{\theta}$ has density $f_\theta(x) = e^{x\theta - b(\theta)}$ with respect to $m$}\right\}.\]
Example of exponential families include the set of Bernoulli distribution, Poisson distributions, Gaussian distribution with known variance or Gamma distributions with known shape parameter. For any exponential family  $\cP$ it can be shown that the mean $\mu(\theta)$ of the distribution $\nu_\theta$ satisfies $\mu(\theta) = \dot{b}(\theta)$. Observe that $\mu$ is a strictly increasing function of the natural parameter $\theta$, hence distributions in $\cP$ can be alternatively parameterized by their means.

We adopt this parameterization in this paper. Letting $\cI := \dot{b}(\Theta)$ be the set of possible mean parameters, for all $\mu \in \cI$ we define $\nu^\mu$ to be the distribution in $\cP$ that has mean $\mu$. We also define the Kullback-Leibler divergence between two distributions in $\cP$ as a function of their means by   

\[d(\mu,\mu') := \text{KL}\left(\nu^\mu,\nu^{\mu'}\right) = \int \ln \frac{f_{\dot{b}^{-1}(\mu)}(x)}{f_{\dot{b}^{-1}(\mu')}(x)} f_{\dot{b}^{-1}(\mu)}(x) \dif m(x).\]
This divergence function has a closed form expression in the classical exponential families mentioned above. For example for Gaussian distribution with variance $\sigma^2$ one has $d(\mu,\mu') = (\mu-\mu')^2/(2\sigma^2)$ and for Bernoulli distributions $d(\mu,\mu') = \mu \ln(\mu/\mu') + (1-\mu)\ln((1-\mu)/(1-\mu'))$. Further examples can be found in \cite{KLUCBJournal}.

An exponential family bandit model is a sequence of $K$ probability distributions $\nu^{\mu_1},\dots,\nu^{\mu_K}$ that belong to some one-dimensional canonical exponential family $\cP$: it can be fully parametrized by the vector of means $\bm\mu = (\mu_1,\dots,\mu_K) \in \cI^K$. In a bandit model, data is collected sequentially: an arm $A_t$ is selected at round $t$ and a sample $X_t$ from the distribution $\nu^{\mu_{A_t}}$ is observed. We denote by $N_a(t) = \sum_{s=1}^t \ind_{(A_s = a)}$ the number of selections of arm $a$ in the first $t$ rounds and $S_a(t) = \sum_{s=1}^t X_s \ind_{(A_s = a)}$ the sum of these observations. The empirical mean of the observations obtained from arm $a$ up to round $t$ is therefore defined as $\hat{\mu}_a(t) = S_a(t) / N_a(t)$ once $N_a(t) \neq 0$. We let $\cF_t = \sigma(A_1,X_1,\dots,A_t,X_t)$ be the filtration generated by the observations gathered within the first $t$ rounds and assume the sampling rule is such that $A_t$ is mesurable with respect to $\sigma(\cF_{t-1},U_t)$ where $U_t$ is a uniform random variable that is independent from $\cF_{t-1}$ (allowing randomized algorithms).

In this paper, our objective is to prove \emph{time-uniform} deviation inequalities for sums involving the terms $N_a(t)d(\hat{\mu}_a(t),\mu_a)$ (or some one-sided versions of these). The price for uniformity in time will be some $\ln\ln(N_a(t))$ term and we shall for example obtain deviation inequalities for sums of the entries of a stochastic process $\bm X(t) = \{X_a(t)\}_{a=1}^K$ of the form
\begin{equation}\label{eq:Xform}
  X_a(t) = N_a(t) d(\hat{\mu}_a(t),\mu_a) - c\ln(d + \ln N_a(t))
\end{equation}
for some constants $c$ and $d$. We now describe a general method to obtain time-uniform deviation inequalities for \emph{any} arm-dependent stochastic process $\bm X(t)$.

\subsection{A General Method for Obtaining Deviation Inequalities}

Let $\bm X(t) = \{X_a(t)\}_{a=1}^K$ be a stochastic process indexed by arms. Here we introduce a central assumption under which it is easy to obtain deviation inequalities for sums of the entries of $\bm X(t)$ by combining Ville's inequality for martingales with the Cramér-Chernoff method. For this reason, we call such processes $g$-VCC (in reference to the Ville-Cramér-Chernoff trio). We will also follow \cite{shafer2011test} in calling any non-negative martingale $M(t) \ge 0$ of unit initial value $M(0) = 1$ a \emph{test martingale}.


\begin{definition}\label{def:Central}
  Let $g : \Lambda \rightarrow \R$ be a function defined on a non-empty interval $\Lambda \subseteq \R$. A stochastic process $\bm X(t) = \{X_a(t)\}_{a=1}^K$ is called \emph{$g$-VCC} if it satisfies the following properties.
\begin{enumerate}
\item For any arm $a$ and $\lambda \in \Lambda$ there exists a test martingale  $M^{\lambda}_a(t)$ such that
 \begin{equation}\label{MartingaleTrick}\tag{$*$}
    \forall t \in \N, \ \ M^\lambda_a(t) \geq e^{\lambda X_a(t)  - g(\lambda)}.
  \end{equation}

\item\label{it:product}
  For any subset $\cS \subseteq \{1,\dots,K\}$ and for any $\lambda \in \Lambda$, the product $\prod_{a \in \cS} M^\lambda_a(t)$ is a martingale.
 \end{enumerate}

\end{definition}

We note that the independent work of \cite{Howard20Survey} also presents a general method based on the Cram\'er-Chernoff method to derive time-uniform concentration inequalities. The authors propose deviation inequalities for a two-dimensional stochastic processes $(S_t,V_t)$ under an assumption that bears similarities with $(*)$: $\exp(\lambda S_t - \phi(\lambda) V_t)$ has to be upper bounded by a martingale, for a known function $\phi$ and for all $\lambda$ in a certain range. Yet the proposed applications of these two general methods differ, in particular there is no emphasis on measuring deviations for multiple arms in the work of \cite{Howard20Survey}. 


\begin{remark}\label{rem:Ideal}
  To calibrate what to expect for $g$, we can use knowledge of the asymptotic distribution of the $X_a(t)$ given in \eqref{eq:Xform}. In our applications, Wilks' phenomenon \citep[see][Chapter~17]{DeLaPenaal09Book} tells us that $2 X_a(t)$ is asymptotically (for $N_a(t) \to \infty$) $\chi^2$ distributed when $c=0$ in \eqref{eq:Xform}. For $2 Y \sim \chi^2$, we have $\bE \sbr{e^{\lambda Y}} = (1-\lambda)^{-1/2}$. This strongly suggests (and this is what we will find) that $g(\lambda)$ should be at least $\frac{1}{2} \ln (1-\lambda)$, plus a mild additional cost for uniformity in time. For this reason we will refer to $g_{\chi^2}(\lambda) = \frac{1}{2} \ln (1-\lambda)$ as the ``ideal function''.
\end{remark}

For a $g$-VCC stochastic process $\bm X(t) = \{X_a(t)\}_{a=1}^K$, we provide a general deviation inequality for the sum of the entries $X_a(t)$ over any subset of arms. The threshold is related to the function $g$ through the following quantities.

\begin{definition} \label{def:Cg} For $g : \Lambda \rightarrow \R^+$, we define for all $x>0$, 
\begin{eqnarray*}
   C^g(x)
  &:=&
  \min_{\lambda \in \Lambda}~
  \frac{g(\lambda) + x}{\lambda}.
\end{eqnarray*}
We also define the convex conjugate of $g$, $ g^*(x)  :=  \max_{\lambda \in \Lambda}~ \left(\lambda x - g(\lambda)\right)$.  
\end{definition}

With these functions in hand, we can now state our $g$-VCC deviation inequality.

\begin{lemma} \label{lem:OneSubset} Fix $\cS \subseteq \{1,\dots,K\}$. Let  $\bm X(t) = \{X_a(t)\}_{a=1}^K$ be a $g$-VCC stochastic process.  Then
\begin{eqnarray*}
  \forall x > 0,&& \ \ \bP \left(\exists t \in \N : \sum_{a\in \cS} X_a(t) \ge |\cS| C^g\left(\frac{x}{|\cS|}\right)\right)  \leq  e^{-x}, \\
  \forall u > 0,&& \ \ \bP\left(\exists t \in \N : \sum_{a\in \cS} X_a(t) > u\right) \leq \exp\left(-|\cS|g^*\left(\frac{u}{|\cS|}\right)\right). 
\end{eqnarray*}
\end{lemma}

\begin{proof} Fix $\lambda \in \Lambda$. As $\bm X(t)$ is $g$-VCC (see Definition~\ref{def:Central}), we find
\begin{eqnarray*}
\bP\left(\exists t \in \N : \sum_{a\in \cS} X_a(t) > u\right) &=& \bP\left(\exists t \in \N : e^{\lambda \left[\sum_{a\in \cS} X_a(t)\right]} > e^{\lambda u}\right) \\
&\leq & \bP\left(\exists t \in \N : \prod_{a \in \cS} M_{a}^{\lambda}(t) > e^{\lambda u - |\cS|g(\lambda)}\right).
\end{eqnarray*}
As $\prod_{a \in \cS}M_{a}^{\lambda}(t)$ is a test martingale, it follows from Ville's inequality ($\bP(\exists t \in \N^* : M(t) \ge 1/x) \leq x$ for any non-negative super-martingale starting from $\bE[M(0)]=1$ and any $x \in (0,1]$, \citealt{Ville39}) that
\begin{equation}\bP\left(\exists t \in \N : \sum_{a\in \cS} X_a(t) > u\right) \leq  e^{-\left[\lambda u - |\cS|g(\lambda)\right]}\label{ToOptim1}\end{equation}
Equivalently, one can also establish that for all $x>0$, for all $\lambda \in \Lambda$, 
\begin{equation}\bP\left(\exists t \in \N : \sum_{a\in \cS} X_a(t) > \frac{|\cS|g(\lambda) + x}{\lambda}\right) \leq  e^{-x}\label{ToOptim2}\end{equation}
Picking the best possible $\lambda$ in \eqref{ToOptim2} yields the first inequality in Lemma~\ref{lem:OneSubset} while picking the best possible $\lambda$ in \eqref{ToOptim1} yields the second inequality. 
\end{proof}

The deviation inequalities given in Lemma~\ref{lem:OneSubset} are either expressed in terms of the threshold function $C^g$ or in terms of the convex conjugate $g^*$. Depending on $g$, one of these two quantities might be easier to compute that the other one. Note that if $g^*$ is well-behaved, the threshold function can be obtained by inverting $g^*$, as stated below.

\begin{proposition} Assume $g^*$ is increasing. For all $u \in g^*(\R^+)$, $C^g(u) = (g^*)^{-1}(u)$.  
\end{proposition}

\begin{proof} As $g^*$ is increasing on $\R^+$, the inverse function $(g^*)^{-1}$ is well defined on $\cI := g^*(\R^+)$. From the definitions of $C^g$ and $g^*$, it is easy to check that 
\[\forall x > 0, \ \ g^*(C^g(x)) \geq x \ \ \text{and} \ \ C^g(g^*(x)) \leq x.\]
These two inequalities respectively yield that for all $u \in \cI$, $(g^*)^{-1}(u) \leq C^g(u)$ and $C^g(u) \leq (g^*)^{-1}(u)$, which concludes the proof.  
\end{proof}


%

\subsection{Mixture Martingales}\label{subsec:MixtureMartingales}

Introducing the cumulant generating function $\phi_{\mu}(\eta) : = \ln \bE_{X \sim \nu_{\mu}}\left[e^{\eta X}\right]$ for all $\mu \in \cI$, it holds for all $\eta \in \R$ that
\begin{equation}\label{eq:Z.a.eta}
  Z_a^\eta(t) ~\df~ \exp\left(\eta S_a(t) - \phi_{\mu_a}(\eta)N_a(t) \right)
\end{equation}
is a test martingale with respect to the filtration $\cF_t$, for any sampling rule. Indeed, when $A_t = a$ we have
$
  \ex\sbrc*{Z_a^\eta(t)}{A_t, \cF_{t-1}}
  =
  Z_a^\eta(t-1)
  \ex\sbrc*{
    e^{\eta X_t - \phi_{\mu_a}(\eta)}
  }{A_t, \cF_{t-1}}
=
Z_a^\eta(t-1)$, and the same trivially holds when $A_t \neq a$. So by the tower rule $\ex\sbrc*{Z_a^\eta(t)}{\cF_{t-1}} = \ex\sbrc*{\ex\sbrc*{Z_a^\eta(t)}{A_t, \cF_{t-1}}}{\cF_{t-1}} = Z_a^\eta(t-1)$. More generally, for any probability distribution $\pi$ on $\eta$, the \emph{mixture martingale}
\begin{equation}\label{MixtureMartingale}
  Z_{a}^{\pi}(t) ~\df~ \int Z_a^\eta(t) \dif\pi(\eta)
\end{equation}
is also a test martingale, as can be seen by applying Tonelli's theorem
\[
  \ex \sbrc*{Z_a^\pi(t)}{A_t,\cF_{t-1}}
  ~=~
  \int
  \underbrace{
    \ex \sbrc*{Z_a^\eta(t) }{A_t,\cF_{t-1}}
    }_{= Z_a^\eta(t-1)}
    d\pi(\eta)
  ~=~
  Z_a^\pi(t-1)
  .
\]
Finally, given a family of priors $\vpi = \set{\pi_a}_{a=1}^K$, the \emph{product martingale} $Z_\cS^{\vpi}(t) \df \prod_{a\in \cS} Z_a^{\pi_a}(t)$ is also a test martingale with respect to $\cF_t$, for any subset $\cS$. Namely, when $A_t \in \cS$ we have
\[
  \ex\sbrc*{Z_\cS^\vpi(t) }{A_t, \cF_{t-1}}
  ~=~
  Z_{\cS \setminus \set{A_t}}^\vpi(t-1)
  \underbrace{
    \ex\sbrc*{Z_{A_t}^{\pi_{A_t}}(t) }{A_t, \cF_{t-1}}
  }_{= Z_{A_t}^{\pi_{A_t}}(t-1)}
  ~=~
  Z_\cS^\vpi(t-1)
  ,
\]
and the same result holds trivially when $A_t \notin \cS$. The martingale property follows by the tower rule.
Hence, a sufficient condition to establish that a stochastic process $\bm X(t)$ is $g$-VCC is to exhibit for all $\lambda \in \Lambda$ a family of priors $\pi_{a,\lambda}$ such that $M_a^\lambda(t):=Z_a^{\pi_{a,\lambda}}(t)$ satisfies \eqref{MartingaleTrick}. This is how we proceed in the next sections. 

\subsubsection{Example of Mixture Martingales} Among the first occurrence of such mixture martingales, one can mention the works of \cite{Darling68,Robbins70LIL} which consider the martingale $\int \exp\left(\eta S_t - \frac{\eta^2\sigma^2}{2}t \right) \dif\pi(\eta)$ where $S_t$ is a sum of $t$ i.i.d.\ standard Gaussian random variables and $\pi$ is a Gaussian prior. This martingale coincides with our $Z_a^\pi(t)$ for a single standard Gaussian arm $a$. The choice of Gaussian prior $\pi$ results in a threshold growing like $\sqrt{t \ln t}$. We use different priors, which asymptote at $\eta=0$, to obtain a deviation inequality for $S_t$ that is uniform in time and compatible with the Law of the Iterated Logarithm: $S_t$ is compared to a threshold that grows like $\sqrt{2t\ln\ln(t)}$.

More broadly, the term mixture martingale can refer to any martingale of the form $\int M^{\eta}(t) d\pi(\eta)$ where $M^{\eta}(t)$ is some martingale (not necessarily $Z_a(t)$) and $\pi$ is some probability distribution (that we call the prior). For example the likelihood ratio martingales introduces by \cite{Lai76} are of this form. Mixture martingale constructions are also at the heart of the game-theoretic approach to probability \citep{philipdawid1999}. The ``method of mixtures'' has then been popularized by \cite{DeLaPenaal04,DeLaPenaal09Book} who use it to prove self-normalized deviation inequalities for general stochastic processes. Examples of its use include the work of \cite{YadLinear11} who propose a self-normalized deviation inequality for a vector-valued martingale applied to the linear bandit problem or that of \cite{Balsubramani15} who derive time-uniform Hoeffding or Bernstein deviation inequalities. Most of these works present mixture martingales with specific choices of continuous priors for which the corresponding mixture can be either computed in closed form or well-approximated. In this paper, we will rely on priors constructed in hierarchical fashion from discrete and continuous ingredients with the goal of obtaining explicit near-optimal thresholds.

\section{New Deviation Inequalities for Exponential Families}\label{sec:expfam.case}

In this section, we first provide a general deviation result that holds for any one-dimensional exponential family and can also accommodate \emph{one-sided deviations} (Theorem~\ref{thm:DevExpo}). Next, we present in Section~\ref{subsec:Improved} tighter deviation inequalities that measure two-sided deviations for the special cases of Gaussian and Gamma distributions. The two sets of results rely on proving that a stochastic process is $g$-VCC for certain functions $g$, which we do by constructing appropriate mixture martingales based on hierarchical priors in Section~\ref{proof:DevExpo}.

\subsection{Main Result}

To state our result, we introduce one-sided versions of the Kullback-Leibler divergence, namely $d^+(u,v)=d(u,v)\ind_{(u \leq v)}$ and $d^-(u,v) = d(u,v)\ind_{(u \ge v)}$. We further introduce the notation
\begin{eqnarray*}
Y_a(t) &:=& [N_a(t)d(\hat{\mu}_a(t),\mu_a) - 3\ln (1+\ln(N_a(t))]^+ \\
Y_a^-(t) &:=& [N_a(t)d^-(\hat{\mu}_a(t),\mu_a) -3\ln (1+\ln(N_a(t))]^+ \\
Y_a^+(t) &:=& [N_a(t)d^+(\hat{\mu}_a(t),\mu_a) - 3\ln (1+\ln(N_a(t))]^+
\end{eqnarray*}
and let $\bm X(t) = \{X_a(t)\}_{a=1}^K$ be a stochastic process such that, for all $a$, either $\forall t, X_a(t)=Y_a(t)$ (for two-sided deviations) or $\forall t, X_a(t) = Y_a^+(t)$ or $\forall t, X_a(t) = Y_a^-(t)$ (for one-sided deviations).

\begin{remark}
  We use as our correction function $3\ln(1+ \ln N_a(t))$, which is vacuous when $N_a(t) = 0$ because $\ln N_a(t) = - \infty$. Most algorithms for bandits avoid considering this situation, and start by pulling all arms once. In some scenarios, especially with many arms, it may be desirable to include the case $N_a(t) = 0$. There is no essential bottleneck, and one could adjust the analysis to, for example, replace it by $3\ln (1+ \ln (1+N_a(t)))$.
\end{remark}


 We provide in Theorem~\ref{thm:DevExpo} below a new self-normalized deviation inequality featuring a calibration function $\cC_{\text{exp}}$. To give the expression of $\cC_{\text{exp}}$, we need to introduce two functions. First for $u \ge 1$ the function $h(u) = u - \ln u$ and its inverse $h^{-1}(u)$. Secondly, the function defined for any $z \in [1,e]$ and $x \ge 0$ by 
\begin{equation}\label{eq:tilde.h}
\tilde h_z(x)
~=~
\begin{cases}
  e^{1/h^{-1}(x)} h^{-1}(x) & \text{if $x \ge h(1/\ln z)$,}
  \\
  z (x-\ln \ln z)
  & \text{otherwise.}
\end{cases}
\end{equation}

%
\begin{theorem}\label{thm:DevExpo}  Let $\cC_{\emph{exp}} : \R^+ \rightarrow \R^+$ be the function defined by
  \begin{equation}\label{eq:T}
    \cC_{\emph{exp}}(x) ~=~ 2 \tilde h_{3/2}\left(\frac{h^{-1}(1+x) + \ln(2\zeta(2))}{2}\right)
  \end{equation}
  where $\zeta(s) = \sum_{n=1}^\infty n^{-s}$. For $\cS$ a subset of arms and $x>0$,
\[\bP\left( \exists t \in \N: \sum_{a \in \cS} X_a(t) \geq |\cS|\cC_{\emph{exp}}\left(\frac{x}{|\cS|} \right)\right) \leq e^{-x}.\]
Moreover, if $X_a(t)$ measures only one-sided deviation (that is for all $a$, $X_a = Y_a^+$ or $X_a = Y_a^-$), the calibration function can be replaced by the smaller $\tilde{\cC}_{\emph{exp}}(x) ~=~ 2 \tilde h_{3/2}\left(\frac{h^{-1}(1+x) + \ln(\zeta(2))}{2}\right)$.
\end{theorem}

 As can be seen in the proof given in Section~\ref{proof:DevExpo}, this result follows by exhibiting \emph{a family of functions} $g_{\xi}$ such that $\bm X(t)$ is $g_\xi$-VCC, applying Lemma~\ref{lem:OneSubset} and then optimizing the parameters to obtain the best possible calibration function. 
Proposition~\ref{prop:Lambert} below (proved in Appendix~\ref{proof:Lambert}) gives a tight bound on the inverse function $h^{-1}$, which yields an upper bound on the calibration function $\cC_{\text{exp}}$. On can easily see that $\cC_{\text{exp}}(x) \sim x$ when $x$ tends to infinity. For $x \geq 5$, a good approximation of the threshold is $\cC_{\text{exp}}(x) \simeq x + 4\ln(1 + x + \sqrt{2x})$, which is slightly larger than the approximation $\simeq x + \ln(x)$ that is added for comparison to Figure~\ref{fig:CompThres}. 

\begin{proposition} \label{prop:Lambert} The function $h$ is increasing on $[1,+\infty[$ and its inverse function, defined on $[1,+\infty[$, satisfies $h^{-1}(x) = -W_{-1}(-e^{-x})$ with $W_{-1}$ the negative branch of the Lambert function. Moreover,
 \[\forall x \geq 1, \ \ h^{-1}(x) \leq x + \ln(x + \sqrt{2(x-1)}).\]
\end{proposition}

%

\subsection{Refined Results for Gaussian and Gamma Distribution} \label{subsec:Improved}

For Gaussian and Gamma distributions, a different martingale construction, explained in detail in Appendix~\ref{sec:nicecase}, permits to establish the following results.

In a bandit model with Gaussian arms with means $\mu_a$ and known variance $\sigma^2$, the associated divergence is $d(\mu,\mu')= \frac{(\mu-\mu')^2}{2\sigma^2}$ and one can prove the following theorem.

\begin{theorem}\label{corr:Gaussian} In a Gaussian bandit model, introducing for all $a$ the process $X_a(t) = N_a(t) d(\hat{\mu}_a(t) , \mu_a) - 2 \ln (4 + \ln N_a(t))$, the stochastic process $\bm X(t)$ is $g_G$-VCC where
 \[\begin{array}{rcl}
g_{G} : ]1/2,1] & \longrightarrow & \R \\
     \lambda & \mapsto &
2\lambda  - 2\lambda  \ln \left(4\lambda \right) + \ln \zeta(2\lambda) - \frac{1}{2}\ln \left(1 - \lambda\right).
   \end{array}\]
Hence, letting $\cC_{\emph{G}} := C^{g_G}$, it follows from Lemma~\ref{lem:OneSubset} that for every subset $\cS$ and $x > 0$,
\[\bP \left(\exists t \in \N : \sum_{a\in \cS} \left[N_a(t)d\left(\hat{\mu}_a(t),\mu_a\right) - 2\ln(4+\ln N_a(t))\right] \geq |\cS| \cC_{\emph{G}}\left(\frac{x}{|\cS|}\right)\right)  \leq  e^{-x}.\]
\end{theorem}

In a bandit model with arms that are Gamma distributed with means $\mu_a$ and known shape parameter $\alpha$, the associated divergence is $d(\mu,\mu') = \alpha\left(\tfrac{\mu}{\mu'} - 1 - \ln \tfrac{\mu}{\mu'}\right)$ and one can prove the following theorem. 

\begin{theorem}\label{corr:Gamma} In a Gamma bandit model, introducing for all $a$ the process $X_a(t) = N_a(t) d(\hat{\mu}_a(t) , \mu_a) - 2 \ln (4 + \ln N_a(t))$, the stochastic process $\bm X(t)$ is $g_\Gamma$-VCC where
 \[\begin{array}{rcl}
g_{\Gamma} : ]1/2,1] & \longrightarrow & \R \\
     \lambda & \mapsto &
                         2 \lambda - 2\lambda  \ln \left(4 \lambda\right) + \ln \zeta(2\lambda) - \ln \left(1 - \lambda\right).
   \end{array}
 \]
Hence, letting $\cC_{\Gamma} := C^{g_{\Gamma}}$, it follows from Lemma~\ref{lem:OneSubset} that for every subset $\cS$ and $x > 0$,
\[\bP \left(\exists t \in \N : \sum_{a\in \cS} \left[N_a(t)d\left(\hat{\mu}_a(t),\mu_a\right) - 2\ln(4+\ln N_a(t))\right] \geq  |\cS| \cC_{\Gamma}\left(\frac{x}{|\cS|}\right)\right)  \leq  e^{-x}.\]
\end{theorem}

\subsection{Discussion}

The three deviation inequalities given Theorems~\ref{thm:DevExpo}, \ref{corr:Gaussian} and~\ref{corr:Gamma} all provide a control of the two-sided deviations of the empirical means from the true means, of the form
\[
 \bP\left(\exists t \in \N : \sum_{a\in \cS} N_a(t)d(\hat{\mu}_a(t),\mu_a) > \sum_{a\in \cS} c\ln(d+\ln(N_a(t))) + |\cS|\cC\left(\frac{x}{|\cS|}\right)\right) \leq  e^{-x}
\]
where $c$ and $d$ are two constants and $\cC(x)$ is a calibration function. For Gaussian or Gamma distributions one can use $c=2, d=4$ while $c=3,d=1$ apply for other one-dimensional exponential families. A more crucial difference is the calibration function $\cC$, which can be set to $\cC_{\text{G}}$ for Gaussian distributions, to $\cC_{\Gamma}$ for Gamma distributions and to $\cC_{\text{exp}}$ in general.

Those three calibration functions are hard to compare at first as they have no closed-form expressions. Equation \eqref{eq:T} provides an explicit expression for $\cC_{\text{exp}}$ but that still requires to numerically invert the function $h$, while $\cC_{\text{G}}$ and $\cC_{\Gamma}$ can be numerically approximated using Definition~\ref{def:Cg} which requires to minimize a convex function. 
In Figure~\ref{fig:CompThres} we compare those three thresholds to the ``ideal'' calibration function $C^{g_{\chi^2}}$ where $g_{\chi^2}(\lambda) = - \frac{1}{2}\ln(1-\lambda)$ (see Remark \ref{rem:Ideal}). We see that that this idealized calibration satisfies $ C^{g_{\chi^2}}(x) \simeq x + \ln(x)$ and that the calibration functions obtained for Gaussian and Gamma distributions are very close to it. The function $\cC_{\text{exp}}$ seems to be off by an additive term of order 10.

\begin{figure}[h]
\centering
\includegraphics[height = 6.5cm]{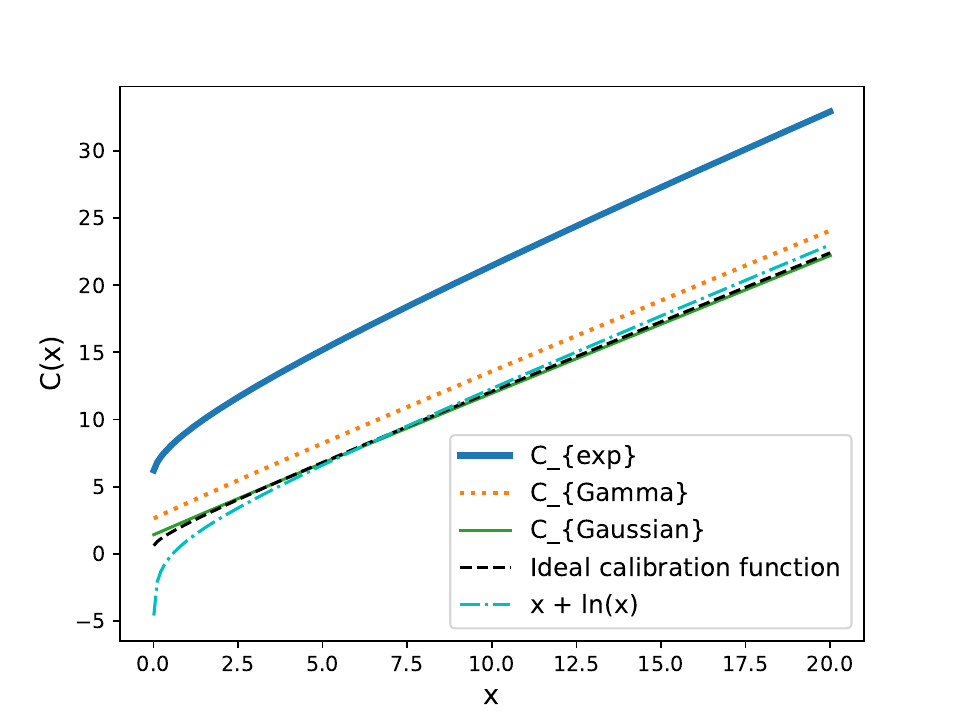} 
\caption{Several calibration functions $\cC(x)$ as a function of $x$. \label{fig:CompThres}}
\end{figure}

Despite this slightly larger calibration function, the general result of Theorem~\ref{thm:DevExpo} is interesting for the following reasons. First, obviously it covers more distributions like Bernoulli and Poisson distributions that are often relevant for applications of multi-armed bandits. Then, we noted that Theorem~\ref{thm:DevExpo} can be made tighter in case only one-sided deviations are measured (when $N_a(t)d^+(\hat{\mu}_a(t),\mu_a)$ or $N_a(t)d^-(\hat{\mu}_a(t),\mu_a)$ are used): $\cC_{\text{exp}}$ can be replaced by the slightly smaller threshold $\tilde{\cC}_{\text{exp}}$. In contrast, the construction presented in Appendix~\ref{sec:nicecase} cannot be easily adapted to obtain better results for one-sided deviations for Gaussian or Gamma distributions. Finally, the presence of the positive part in the definition of $Y_a(t)^{\pm}$ leads to the following improved result  holding uniformly over subsets:

\vspace{-0.3cm}

{\small
\[
 \bP\left(\exists t \in \N : \exists \cS' \subseteq \cS, \sum_{a\in \cS'} N_a(t)d^\pm(\hat{\mu}_a(t),\mu_a) > \sum_{a\in \cS'} 3\ln(1+\ln(N_a(t))) + |\cS|\cC_{\text{exp}}\left(\frac{x}{|\cS|}\right)\right) \leq  e^{-x}.
\]}

\subsubsection{Comparison to the State-of-the-Art} To the best of our knowledge, the only available result that controls deviations over multiple arms simultaneously is the one of \cite{Combes14Lip}. More precisely, Theorem~2 in \cite{Combes14Lip} can be rephrased as follows\footnote{The result only considers Bernoulli arms and $\cS = [K]$, but their analysis can be easily extended to cover the more general case of exponential families and any subset $\cS$}, introducing the function $\tilde f(u) = u - 2\ln(u)$ for $u \geq 2$:
\begin{equation}\bP\left(\exists t \leq n: \sum_{a \in \cS} N_a(t)d^+(\hat{\mu}_a(t),\mu_a) \geq  |\cS|\tilde f^{-1}\left( 1 + \ln\ln(n) + \frac{x+1}{|\cS|}\right) \right)  \leq  e^{-x}.\label{ineq:Combes}
\end{equation}
We note that here the deviations are uniform over a finite time range $\{1,\dots,n\}$. This style of deviation inequalities can also be deduced from Theorem~\ref{thm:DevExpo} for general exponential families:
\begin{equation}
\label{cor:BoundedTime}
\bP\left( \exists t \leq n: \ \ \sum_{a \in \cS} N_a(t) d^+(\hat{\mu}_a(t),\mu_a) \geq 3|\cS|\ln(1+\ln(n)) + |\cS|\cC_{\text{exp}}\left(\frac{x}{|\cS|} \right)\right) \leq e^{-x},
\end{equation}
or from Theorem~\ref{corr:Gaussian} and Theorem~\ref{corr:Gamma}:
\begin{equation}
\label{cor:BoundedTimeGG}
\bP\left( \exists t \leq n: \ \ \sum_{a \in \cS} N_a(t) d^+(\hat{\mu}_a(t),\mu_a) \geq 2|\cS|\ln(4+\ln(n)) + |\cS|\cC\left(\frac{x}{|\cS|} \right)\right) \leq e^{-x},
\end{equation}
where $\cC(x) =\cC_{\text{G}}(x)$ for Gaussian distributions and $\cC(x) = \cC_{\Gamma}(x)$ Gamma distributions. In Figure~\ref{fig:Combes}, we plot the thresholds (right-hand side of the deviation inequalities) featured in \eqref{ineq:Combes} and \eqref{cor:BoundedTime} for different values of $n$, $|\cS|$ and $x$, revealing that the threshold in \eqref{cor:BoundedTime} can be much smaller. 

\begin{figure}[h]
\begin{center}
\includegraphics[width=0.5\linewidth]{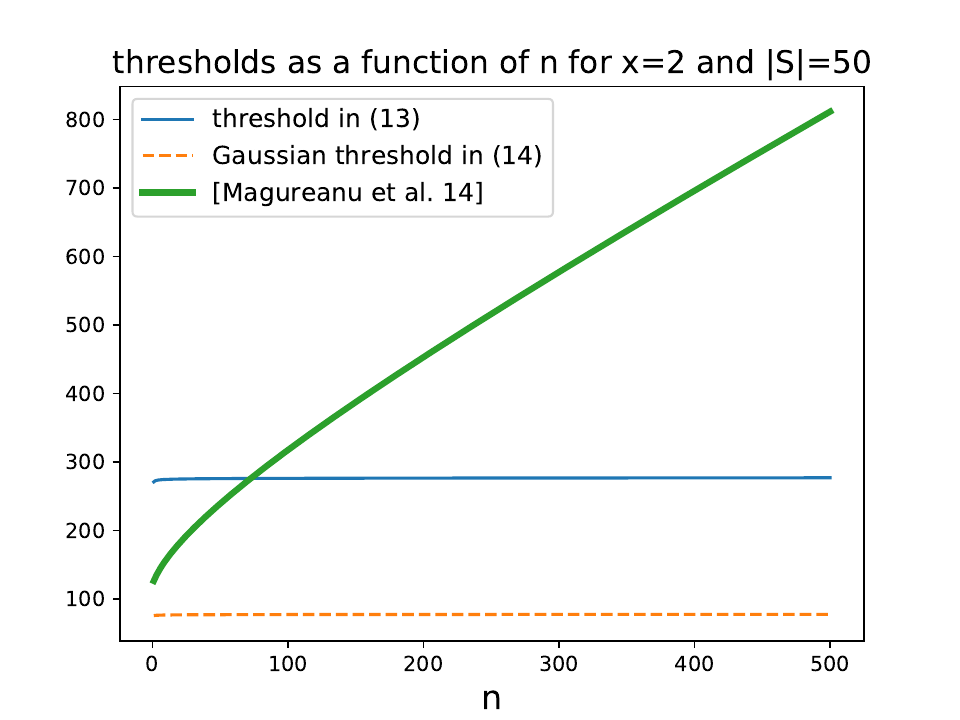}
\end{center}
\includegraphics[width=0.49\linewidth]{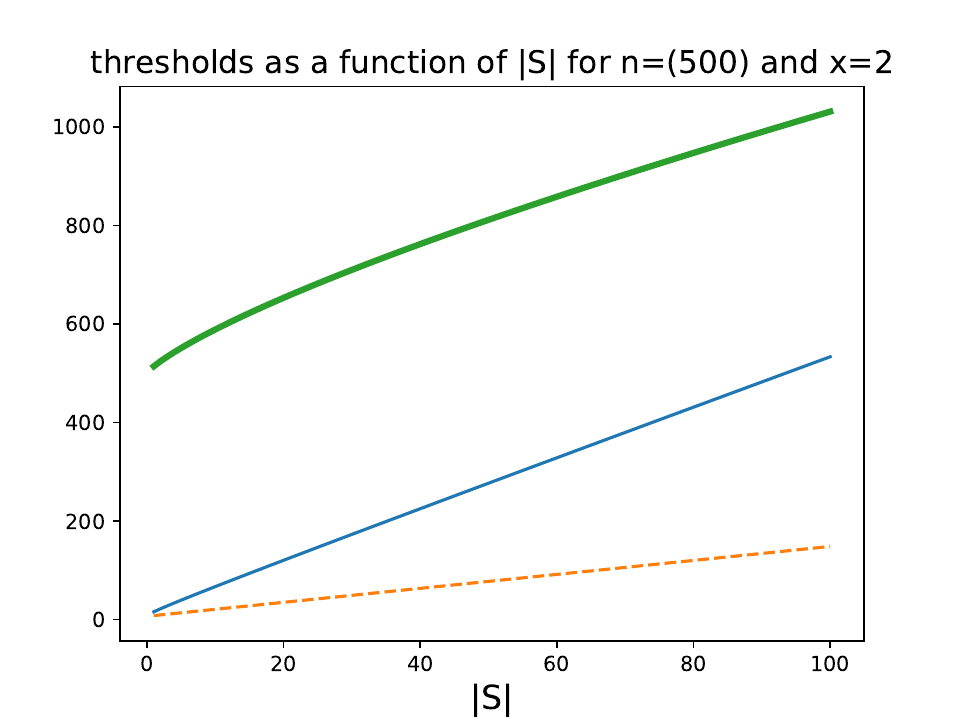}
\includegraphics[width=0.49\linewidth]{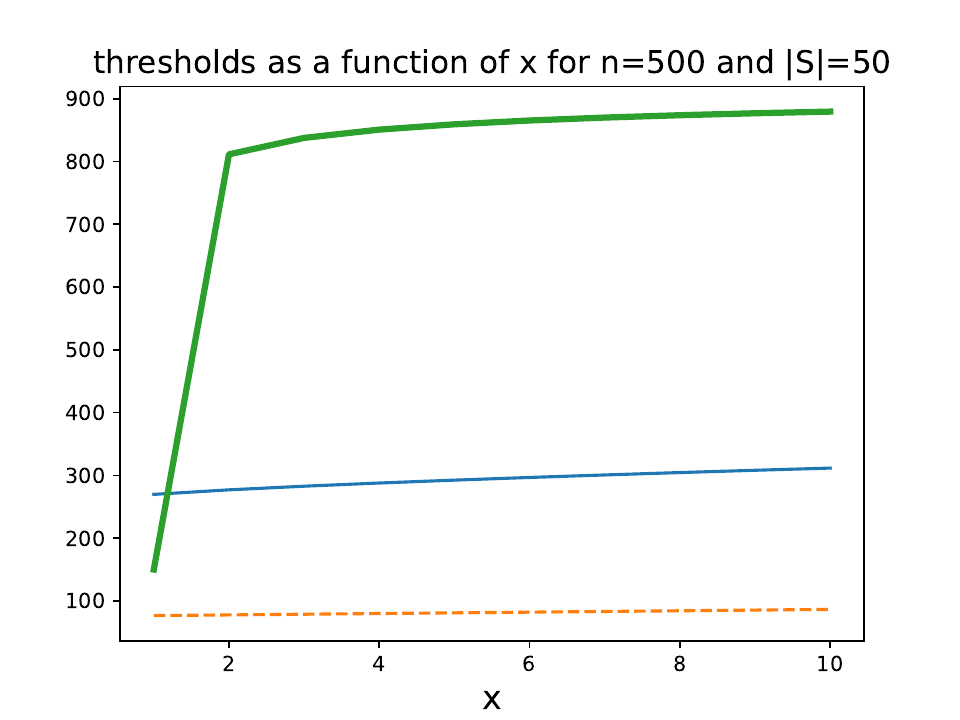}
\caption{Thresholds that follow from Theorem~\ref{thm:DevExpo} and Theorem~\ref{corr:Gaussian} compared to the threshold in \eqref{ineq:Combes} obtained by \cite{Combes14Lip}\label{fig:Combes}.}
\end{figure}


\subsubsection{Improved Result when $|\cS| = 1$} Theorem~\ref{thm:DevExpo} can be made slightly tighter for a subset of size 1 (see Appendix~\ref{proof:onearm}) and we obtain, with $\hat{\mu}_t$ the empirical mean of $t$ observations from a distribution with mean $\mu$ in an  exponential family,
\begin{equation}\bP\left(\exists t \in \N : t d\left(\hat{\mu}_t, \mu\right) \geq 3\ln(1+\ln(t)) + 2\tilde{h}_{3/2}\left(\frac{x + \ln(2\zeta(2))}{2}\right)\right) \leq e^{-x}.\label{dev:OneArm}\end{equation}
Albeit not our main focus, this result is interesting as it provides a deviation result in which the threshold features a $\log(\log(t))$ term, which is known to be unavoidable in the Gaussian case due to the Law of Iterated Logarithm. In the Gaussian (or sub-Gaussian) case, several results of this type already exist \citep{Robbins70LIL,Jamiesonal14LILUCB,JMLR15,ZhaoZSE16,Howard18Bernstein} and we do not claim ours to be the tightest in general. 
Beyond Gaussian distribution, \eqref{dev:OneArm} is the first result that controls deviations uniformly in $t\in \N$ for general exponential families and with a threshold scaling in $\log(\log(t))$. 

\subsection{Sketch of Proofs} \label{proof:DevExpo}

In this section, we provide a detailed proof of Theorem~\ref{thm:DevExpo}, leaving the proof of some intermediate lemmas to Appendix~\ref{details:Expo}. The proofs of Theorem~\ref{corr:Gaussian} and \ref{corr:Gamma} are given in Appendix~\ref{sec:nicecase}, but we provide below a high-level description of the martingale that is used for these results.

\paragraph{Proof of Theorem~\ref{thm:DevExpo}}

Fix $\xi > 0$ and define for all $\lambda \in [0, 1/(1+\xi))$,
\[ g_{\xi}(\lambda)  =  \lambda (1+\xi)\ln\left(C(\xi)\right) - \ln(1-\lambda(1+\xi)) \ \ \ \text{with} \ \ \ C(\xi) = \frac{2\zeta(2)}{(\ln(1+\xi))^2}\;.\]
The proof hinges on the fact that for the stochastic process $\bm X$, there exists a martingale satisfying~\eqref{MartingaleTrick}.

\begin{lemma}\label{thm:AssumptionExpo} For $\xi \in [0,1/2]$, $\bm X$ is $g_\xi$-VCC (see Definition~\ref{def:Central}).
\end{lemma}

As will be seen in the proof of Lemma~\ref{thm:AssumptionExpo}, given shortly, in case the stochastic process $\bm X$ only measures one-sided deviations, that is for all $a$ either $X_a(t) = Y_a^-(t)$ or $X_a(t) = Y_a^+(t)$, then $C(\xi)$ can be replaced by the smaller $C(\xi) = {\zeta(2)}/{(\ln(1+\xi))^2}$: the factor 2 that is removed corresponds to picking a one-sided versus a two-sided prior and leads to the slightly smaller threshold $\tilde{\cC}_{\text{exp}}$ given in the second statement of Theorem~\ref{thm:DevExpo}. 

Using Lemma~\ref{lem:OneSubset}, we obtain the following deviation inequality expressed with the function $C^{g_{\xi}}$ associated to $g_\xi$ (see Definition~\ref{def:Cg}): for all $\xi >0$, for all $x>0$,
\[\bP\left(\exists t \in \N : \sum_{a \in \cS} X_a(t) \geq |\cS| C^{g_\xi}\left(\frac{x}{|\cS|}\right)\right) \leq e^{-x}.\]
Recalling that $C^{g_\xi}(x) = \min_{\lambda \in [0,1/(1+\xi))} \frac{x + g_\xi(\lambda)}{\lambda}$, the proof is completed by applying Lemma~\ref{lem:tuning.tight} below, proved in Appendix~\ref{sec:tuning}, to compute the optimal tuning of $\xi \in [0,1/2]$.

\begin{lemma}\label{lem:tuning.tight}
  Let $C(\xi) = \frac{2 \zeta(2)}{(\ln (1+\xi))^2}$. Fix $z \in [0, e-1]$ and $x \ge 0$. Then
\[
  \inf_{\substack{\xi \in [0, z] \\ \lambda \in [0, 1/(1+\xi)]}}~
  \frac{
    x
    - \ln \del*{1-\lambda (1+\xi)}
  }{
    \lambda
  }
  + (1+\xi)\ln C(\xi)
  ~=~
  2 \tilde h_{1+z} \del*{\frac{h^{-1} \del*{1 + x}
      + \ln \del*{2 \zeta(2)}
    }{2}
  }
  .
\]
\end{lemma}
\qed 

\vspace{-0.3cm}

We now provide the proof of the crucial Lemma~\ref{thm:AssumptionExpo}. 

\paragraph{Proof of Lemma~\ref{thm:AssumptionExpo}: building the martingale} Lemma~\ref{lem:ExistsMart} below shows that the deviations of $X_a(t)$ can be related to the deviations of a well-chosen mixture martingale $Z_a^{\pi}(t)$, where $\pi$ has a discrete support. The proof of Lemma~\ref{lem:ExistsMart} is given in Appendix~\ref{details:Expo}.


\begin{lemma}[mixture martingales]\label{lem:ExistsMart} Fix $\xi \in ]0,1/2[$ and $x>0$. There exists a (discrete) prior $\pi(x) = \pi(x,\xi)$ such that the corresponding mixture martingale \eqref{MixtureMartingale}, denoted by $Z_{a}^{\pi(x)}(t)$, satisfies, for all $t \in \N$,
\[\left\{X_a(t) - (1+\xi)  \ln\left(\frac{2\zeta(2)}{(\ln (1+\xi))^2}\right) \geq x\right\} \subseteq \left\{Z_{a}^{\pi(x)}(t) \geq e^{\frac{x}{1+\xi}}\right\}.\]
If $X_a(t) = Y_a^+(t)$ or $X_a(t)=Y_a^-(t)$, there exists a prior $\pi(x)$ such that 
\begin{equation}\left\{X_a(t) - (1+\xi)  \ln\left(\frac{\zeta(2)}{(\ln (1+\xi))^2}\right) \geq x\right\} \subseteq \left\{Z_{a}^{\pi(x)}(t) \geq e^{\frac{x}{1+\xi}}\right\}.\label{ImprovedOneSided}\end{equation}
\end{lemma}

\bigskip\noindent
A consequence of Lemma~\ref{lem:ExistsMart} is that, for every $z>1$, and every $\lambda>0$ 
\begin{eqnarray*}\left\{e^{\lambda\left({X}_a(t)-(1+\xi)\ln C(\xi)\right)}  \geq z \right\} &\subseteq& \left\{Z_a^{\pi(\ln(z)/\lambda)}(t) \geq e^{\frac{\ln(z)}{\lambda(1+\xi)}}\right\} \\
&\subseteq& \Big\{\underbrace{Z_{a}^{\pi(\ln(z)/\lambda)}(t) e^{-\frac{\ln(z)}{\lambda(1+\xi)}}}_{:=W_a^{z,\lambda}(t)}\geq 1\Big\},
\end{eqnarray*}
where $W_a^{z,\lambda}(t)$ is a martingale that satisfies $\bE[W_a^{z,\lambda}(0)] = e^{-\frac{\ln(z)}{\lambda(1+\xi)}}$ and, due to the above inclusion, \begin{equation}W_a^{z,\lambda}(t) \geq \ind_{\left(e^{\lambda\left(X_a(t)-(1+\xi)\ln C(\xi)\right)} \geq z \right)}.\label{eq:UBMart}\end{equation}
We now define another mixture martingale, for $\lambda \in \left]0, \frac{1}{1+\xi}\right[$:
\begin{eqnarray*}
 W_a^{\lambda}(t) = 1 + \int_{1}^\infty W_a^{z,\lambda}(t)dz.
\end{eqnarray*}
Using inequality \eqref{eq:UBMart} yields
\[W_{a}^{\lambda}(t) \geq e^{\lambda \left(X_a(t)-(1+\xi)\ln C(\xi) \right)}.\]
Moreover, a direct computation shows that $W_a^{\lambda}(0) = \frac{1}{1-\lambda(1+\xi)}$. Finally defining \[M_a^{\lambda}(t) ={(1 - \lambda(1+\xi))} W_a^{\lambda}(t),\]
one has that $M_a^{\lambda}(t)$ is a test martingale, i.e.\  $\bE[M_{a}^{\lambda}(t)] = 1$, that satisfies
\begin{eqnarray*}
 M_a^{\lambda}(t) &\geq & \exp\left(\lambda X_a(t) - \lambda (1+\xi)\ln(C(\xi))+ \ln(1-\lambda(1+\xi)) \right) \\
 & = & \exp\left(\lambda X_a(t) - g_{\xi}(\lambda)\right),
\end{eqnarray*}
which concludes the proof. Note that if for all $a$, $X_a(t) = Y_a^\pm(t)$, using the tighter statement \eqref{ImprovedOneSided} allows to replace the constant $C(\xi)$ by the smaller value $\frac{\zeta(2)}{(\ln (1+\xi))^2}$.

Above, we are in essence building a test martingale of value $M_t \ge e^{\lambda X_t}$ from test martingales guaranteeing $Z_t \ge e^{x} \ind \set{X_t \ge x}$. The possibilities and limits of doing this are exactly characterised by \cite{DAWID2011157} in the process of characterising the so-called \emph{admissible capital calibrators}. By changing the mixture on thresholds $x$ from exponential (as we do here) to polynomial, it is technically possible to guarantee $M_t \ge e^{X_t - O(\ln X_t)}$. We do not pursue this direction, as the additional $\ln X_t$ is not convenient for combining evidence of arms, and moreover it is not at all clear that the high cost in terms of multiplicative constants (i.e.\ the $g(\lambda)$) is worth it.
\qed

\vspace{-0.3cm}

\subsubsection{Alternative Mixture Martingales} The martingale $M_a^{\lambda}(t)$ built in the proof of Lemma~\ref{thm:AssumptionExpo} can be viewed as a mixture martingale with a hierarchical prior, which is a continuous mixture of some discrete priors $\pi(\cdot)$ defined in Lemma~\ref{lem:ExistsMart}. Indeed, one can write
\[M_a^{\lambda}(t) = (1-\lambda(1+\xi))Z_a^{0}(t) + \int_{1}^{\infty}\left(\int Z_a^{\eta}(t) \pi(\log(z)/\lambda)(d\eta)\right)(1-\lambda(1+\xi))e^{-\frac{\log(z)}{\lambda(1+\xi)}}dz.\]

To prove Theorem~\ref{corr:Gaussian} and Theorem~\ref{corr:Gamma}, we build different mixture martingales in Appendix~\ref{sec:nicecase}. Interestingly, they also rely on a hierarchical prior but this time the prior is a discrete average of continuous priors. More precisely, the martingale used in each case can be written of the form 
\[M_a^{\lambda}(t) = \sum_{i=1}^{\infty} \gamma_i \int Z_a^{\eta}(t)p_{T_i}^{\lambda,\mu_a}(\eta) d\eta\]
for a well chosen family of $(\gamma_i,T_i) \in (\R^+,\N^*)$, where $p_{t}^{\lambda,\mu_a}(\eta)$ is a continuous function satisfying 
\[\forall x, \ \int  e^{\eta (tx) - \phi_{\mu}(\eta) t} p_n^{\lambda,\mu_a}(\eta) d\eta = e^{\lambda t d(x,\mu_a)}.\]

We see that in both cases, elementary non-negative martingales of the form $Z_a^\eta(t)$ from \eqref{eq:Z.a.eta} are mixed under a (hierarchically constructed) prior distribution on $\eta$. Both approaches are similar in spirit, both implementing the Laplace technique of achieving a value close to that of the maximiser $\hat \eta$ (which is a function of the random data), and the peeling technique (to adapts to the random sample size $N_a(t)$). The combined density has an asymptote at $\eta = 0$, with density nearby proportional to $\frac{1}{\eta (\ln \eta)^2}$. The log factor is necessary for making the prior proper, and it is also the reason for the $\ln \ln N_a(t)$ terms in our deviation inequality.

An interesting alternative approach, which goes slightly outside the mixture martingale framework, is taken by \cite{squint}. There, a mixture martingale is constructed by mixing the increment $Z_a^\eta(t) - Z_a^\eta(0)$ under the \emph{improper} prior with density $1/\eta$. Subtracting $Z_a^\eta(0)$ solves the problem that, without it, the improper mixture would be infinite. However, it has the effect that the mixture can become negative, interfering with Ville's inequality. Yet the mixture value can be shown (for bounded outcomes) to be at least $- \ln N_a(t)$. Taking that into account properly yields, again, the $\ln \ln N_a(t)$ term in the resulting deviation inequality. 
We believe exploring these ideas further to be a worthwhile direction for future research.

\section{Asymptotically Optimal Adaptive Sequential Testing}\label{sec:generic.solved}

In this section, we explain how our new deviation inequalities can be useful to prove the correctness of a stopping strategy for generic sequential adaptive hypothesis testing problems, that we refer to as \emph{sequential identification problems}. 

Given a bandit model, we consider $M$ hypotheses $\cH_1 = (\bm \mu \in \cO_1), \dots, \cH_M = (\bm \mu \in \cO_M)$ where $\cO_1, \dots, \cO_M$ are open sets forming a partition of the set of possible means $\cO$. Our goal is to adaptively sample the arms until a decision is made that one of the hypotheses $\ihat$ is correct. Our goal is to identify the correct hypothesis for all possible means $\bm \mu \in \cO$. More precisely, we aim for \emph{$\delta$-correct strategies}, for which $\forall\bm \mu \in \cO, \ \bP_{\bm\mu}\left(\bm \mu \in \cO_{\ihat}\right) \geq 1 - \delta$. This problem falls into the framework of Sequential Adaptive Hypothesis Testing as introduced by \cite{Chernoff59} --who studied only discrete hypotheses and considered a different performance metric-- and is called General-Samp by \cite{ChenGLQW17}, who study Gaussian arms with unit variance. 

For general exponential family bandits, we analyse below a natural stopping rule based on Generalized Likelihood Ratio (GLR) tests. We prove that this stopping rule is $\delta$-correct for any sequential identification problem and that in some cases it attains the minimal sample complexity (in a regime of small risk $\delta$) when coupled with an appropriate sampling rule. We note that the independent work of \cite{Juneja19} studies the same problem as ours under the name ``partition identification'', also for exponential families. However, that work puts less emphasis on stopping rules, and uses the deviation inequality of \cite{Combes14Lip} for its analysis.


\subsection{A General Stopping Rule}

For every $\bm \mu$, we define \[\mathrm{Alt}(\bm \mu) = \bigcup_{i : \bm \mu \notin \cO_i} \cO_i.\]
If $\bm \mu \in \cO$, we let $i^*(\vmu)$ be the index of the unique element in the partition to which $\bm \mu$ belongs; in particular $\vmu \in \cO_{i^*(\vmu)}$ and $\mathrm{Alt}(\bm \mu) = \cO\backslash \cO_{i^*(\vmu)}$.
We let $\hat{\bm\mu}(t)$ be the vector of empirical means of the arms based on the observations available up to round $t$. If $\hat{\bm{\mu}}(t) \in \cO$, we let $\ihat(t) = i^*(\hat \vmu(t))$ so that $\hat{\bm\mu}(t) \in \cO_{\ihat(t)}$.

\begin{definition} The GLR statistic is defined as
\begin{equation}\label{eq:GLR}
  \hat{\Lambda}_t = \inf_{\bm\lambda \in \Alt(\hat{\bm \mu}(t))} \sum_{a=1}^K N_a(t) d\left(\hat{\mu}_a(t),\lambda_a\right).
\end{equation}
Given a sequence of thresholds $(\hat c_t(\delta))_{t\in\N}$, the GLR stopping rule with thresholds $\hat c_t(\delta)$ is defined by
\begin{equation}\label{eq:GLR.stop}
  \tau_\delta := \inf \left\{ t \in \N : \hat{\Lambda}_t > \hat{c}_t(\delta)\right\}.
\end{equation}
\end{definition}

A Generalized Likelihood Ratio statistic is usually defined for testing a possibly composite hypothesis $\cH_0 :(\mu\in \Omega_0)$ against a possibly composite alternative $\cH_1 : (\mu \in \Omega_1)$ by
\[R_t = \frac{\sup_{\lambda \in \Omega_0 \cup \Omega_1}\ell(X_1,\dots,X_t ; \lambda)}{\sup_{\lambda \in \Omega_0}\ell(X_1,\dots,X_t ; \lambda)}, \]
where $X_1,\dots,X_t$ are some observations whose likelihood $\ell(X_1,\dots,X_t ; \mu)$ depends on some unknown parameter $\mu$. Large values of $R_t$ tend to reject the hypothesis $\cH_0$. When the observations are obtained under a sampling rule $(A_t)$ in an exponential family bandit model and $\hat{\bm\mu}(t) \in \Omega_0\cup\Omega_1$ it can be shown that
\[\ln (R_t) = \inf_{\bm\lambda \in \Omega_0} \sum_{a=1}^K d(\hat{\mu}_a(t),\lambda_a).\]
The GLR statistic $\hat{\Lambda}_t$ can thus be interpreted as a statistic for testing $\cH_0 : \set*{\bm \mu \in \Alt(\hat{\bm \mu}(t))}$ against $\cH_1 : \set*{\bm \mu \in \cO_{\ihat(t)}}$ (if $\hat{\bm\mu}(t) \in \cO$, otherwise note that $\hat{\Lambda}_t = 0$ which prevent from stopping). However the two hypotheses that are ``tested'' at time $t$ are data-dependent.  Still, large values $\hat{\Lambda}_t$ tend to reject $\left(\bm \mu \in \Alt(\hat{\bm \mu}(t))\right)$: hypothesis $\ihat(t)$ must be true. Another possible interpretation of the GLR stopping rule is that it is running in parallel $M$ GLR tests of $\cH_0 : (\bm\mu \in \cO \backslash \cO_i)$ against $\cH_1 : (\bm\mu \in \cO_i)$ for $i=1,\dots,M$ and stops the first time one of these tests $\hat{\imath}$ rejects $\cH_0$. This ``Parallel GLRT'' view is the one discussed for example by \cite{GK19Epsilon}.

It can be observed that $\left\{\hat{\Lambda}_t > \hat{c}_t(\delta)\right\} = \left\{\cC_t(\delta)\subseteq \cO_{\ihat(t)}\right\}$ where $\cC_t(\delta)$ is the \emph{confidence region}
\begin{equation}\cC_t(\delta) := \left\{\bm \lambda : \sum_{a=1}^K N_a(t) d (\hat{\mu}_a(t),\lambda_a) \leq \hat{c}_t(\delta)\right\}.\label{CI:First}\end{equation}
The GLR stopping rule \eqref{eq:GLR.stop} can thus be rephrased in the following way: stop when the set of statistically plausible parameters $\cC_t(\delta)$ is included in one fold of the partition. Building on Theorem~\ref{thm:DevExpo}, Proposition~\ref{prop:GLRAnalysis} below provides a choice of thresholds for which the  GLR stopping rule yields a $\delta$-correct algorithm. We provide a choice of thresholds for which the GLR rule is $\delta$-correct when the hypothesis $\cH_{\ihat(\tau)}$ is recommended and the corresponding confidence intervals $\cC_t(\delta)$ always contain the true parameter with probability larger than $1-\delta$.

\begin{proposition}\label{prop:GLRAnalysis} Let $\cC_{\emph{exp}}$ be the threshold function defined in Theorem~\ref{thm:DevExpo}. The sequence of thresholds
\begin{equation}\hat{c}_t(\delta) = 3 \sum_{a = 1}^K \ln(1 + \ln N_a(t)) + K \cC_{\emph{exp}}\left(\frac{\ln(1/\delta)}{K}\right)\label{UniversalThreshold}\end{equation}
is such that, \emph{for every sampling rule}, \[\bP_{\bm\mu}(\forall t \in \N, \bm \mu \in \cC_t(\delta)) \geq 1 - \delta \ \ \ \text{and}  \ \ \ \bP_{\bm\mu}\left(\tau_\delta < \infty, \ihat(\tau_\delta) \neq i^*(\bm\mu)\right) \leq \delta.\]
\end{proposition}

\begin{proof} Using Theorem~\ref{thm:DevExpo} in the last inequality, one can write
\begin{eqnarray*}
 \bP_{\bm\mu}\left(\tau < \infty, \ihat(\tau) \neq i^*\right) & \le & \bP_{\bm\mu}\left(\exists t \in \N : \ihat(t) \neq i^*, \hat{\Lambda}_t > \hat c_t(\delta)\right) \\
                                                              & = & \bP_{\bm\mu}\left(\exists t \in \N : \exists i \neq i^*, \cC_t \subseteq \cO_i\right) \\
  &\leq &  \bP_{\bm\mu}\left(\exists t \in \N : \bm \mu \notin \cC_t\right) \\
 & = & \bP_{\bm\mu}\left(\exists t \in \N : \sum_{a=1}^K N_a(t) d(\hat{\mu}_a(t),\mu_a) \geq \hat c_t(\delta)\right) \\
 & \leq & \delta.
\end{eqnarray*}
This proves both claims of Proposition~\ref{prop:GLRAnalysis}.
\end{proof}

\subsection{An Asymptotically Optimal Adaptive Testing Procedure}\label{sec:asy.opt}

Proposition~\ref{prop:GLRAnalysis} provides a threshold for which the GLR stopping rule \eqref{eq:GLR.stop} is $\delta$-correct \emph{for any sampling rule}. We now show that used in conjunction with an appropriate ``Tracking'' stopping rule, it can even attain the optimal sample complexity. The following lower bound generalizes the sample complexity lower bound obtained by \cite{GK16} for the particular Best Arm Identification problem and is obtained with the exact same change-of-measure technique.

\begin{proposition}\label{prop:LB} Define the complexity term $T^*(\bm \mu)$ as  
\[T^*(\bm\mu)^{-1}= \sup_{\bm w \in \Sigma_K} \inf_{\bm \lambda \in \Alt(\bm\mu)} \sum_{a=1}^Kw_a d(\mu_a,\lambda_a),\]
where $\Sigma_K = \left\{ \bm w \in [0,1]^K : \sum_{i=1}^K w_i = 1\right\}$. Then any $\delta$-correct strategy satisfies \[\forall \bm\mu \in \cO,  \ \bE_{\bm\mu}[\tau_{\delta}] \geq T^*(\bm\mu) \ln \left(\frac{1}{3\delta}\right).\]
\end{proposition}

We define, when they exist (that is, when the argmax below is unique) the \emph{optimal weights}
\begin{equation}\label{eq:wstar}
  \w^*(\bm\mu):= \argmax{\bm w \in \Sigma_K} \ \inf_{\bm \lambda \in \Alt(\bm\mu)} \sum_{a=1}^Kw_a d(\mu_a,\lambda_a)
\end{equation}
for $\bm \mu \in \cO$. For well-behaved sequential testing problems, those weights indicate the fraction of samples that should be allocated to each arm by an optimal strategy. This motivates the Tracking rule, originally proposed by \cite{GK16} as the D-Tracking rule for Best Arm Identification and that we recall here.  Letting $\cU_t = \{ a \in \{1,\dots,K\} : N_a(t)  \leq \max(\sqrt{t} - K/2,0)\}$ be the set of under-sampled arms, at time $t+1$ the selected arm is
\begin{equation}\label{eq:tracking.rule}
A_{t+1} \in \left\{\begin{array}{lll}
                       \displaystyle
                       \argmin{a \in \cU_t}  \ N_a(t)  \ \ \ & \text{if} \ \cU_t \neq \emptyset & (\textit{forced  exploration})\\
                       \displaystyle
                    \argmax{a \in [K]} \ t \,w_a^*(\hat{\bm\mu}(t)) -{N_a}(t) &\text{o.w.}& (\textit{tracking the plug-in estimate})
                   \end{array}
                 \right.
\end{equation}
It can be noted that $\bm w^*(\hat{\bm\mu}(t))$ is defined only if $\hat{\bm\mu}(t) \in \cO$. In practice if $\hat{\bm\mu}(t) \notin \cO$ the tracking step of the algorithm can be replaced by uniform exploration. Due to the forced exploration, as $\bm\mu \in \cO$ (which is an open set by assumption) the law of large numbers ensures that at some point $\hat{\vmu}(t) \in \cO$, and the tracking step can be applied.

\begin{theorem}\label{thm:OptimalTesting} Assume that the following assumptions are satisfied: 
\begin{enumerate}
\item\label{it:unique}
  For every $\bm \mu \in \cO$, there is a unique vector of optimal weights $\w^*(\bm\mu)$
 \item\label{it:cont}
   For all $i\in \{1,\dots,M\}$, the mapping $\bm \mu \mapsto \w^*(\bm \mu)$ is continuous on $\cO_i$.
\end{enumerate}
For $\delta \in (0,1]$ let  $\hat{c}_t(\delta)$ be a deterministic sequence of thresholds that is increasing in $t$ and for which there exists constants $C,D>0$ such that \[ \forall t \geq C, \forall \delta \in (0,1], \ \hat{c}_t(\delta) \leq \ln \left(\frac{Dt}{\delta}\right).\] Let $\tau_\delta$ be the GLR stopping rule \eqref{eq:GLR.stop} with thresholds $\hat{c}_t(\delta)$. The Tracking rule \eqref{eq:tracking.rule} ensures
\[\limsup_{\delta \rightarrow 0} \frac{\bE_{\bm\mu}\left[\tau_\delta\right]}{\ln(1/\delta)} = T^*(\bm\mu).\]
\end{theorem}

The proof of Theorem~\ref{thm:OptimalTesting} is given in Appendix~\ref{proof:SC}. Combining this result with Proposition~\ref{prop:GLRAnalysis} yields that an adaptive sequential test using the Tracking rule and the GLR stopping rule with thresholds \eqref{UniversalThreshold} is $\delta$-correct for every $\delta \in (0,1]$ and its sample complexity is asymptotically matching the lower bound of Proposition~\ref{prop:LB}, provided that the optimal weights $\w^*(\bm\mu)$ are well defined and continuous in $\bm\mu$.

Efficient ways to compute those weights are also needed for the actual implementation of the Tracking rule. In the next section, we will discuss particular examples of adaptive sequential tests in which those requirements are fulfilled and optimal (and efficient) adaptive testing is thus possible. We will see that smaller thresholds than the universal threshold \eqref{UniversalThreshold} can be used in some cases. 

The assumptions of Theorem~\ref{thm:OptimalTesting} are frequently satisfied for practical problems (see also \citealt[][Lemma~1]{NIPS2017_6773} proving continuity of the highly related oracle regret problem for the structured multi-armed bandit problem in the fixed-budget setting, under a unique optimiser assumption similar to~\ref{it:unique}). Uniqueness may however fail for other practical problems, including e.g.\ the Minimum Threshold problem studied by \cite{kaufmann2018sequential}. A solution for such cases was proposed in recent follow-up work by \cite{multiple.answers}, who propose regarding the oracle weight map $\vmu \mapsto \w^*(\vmu)$ from \eqref{eq:wstar} as set-valued, and prove that it is upper hemi-continuous and convex-valued for every sequential identification problem of the form we consider here (in particular with a unique correct answer for each instance). Leveraging these two properties, they analyse a variation of the Tracking rule~\eqref{eq:tracking.rule} for which the overall approach is asymptotically optimal in general \citep[Theorem~7]{multiple.answers}.

\section{Smaller Thresholds for Better Sequential Tests}\label{sec:stop.for.test}

A stylized form of (two-sided) deviation inequalities obtained in this paper (in Corollaries~\ref{corr:Gaussian} and~\ref{corr:Gamma} and Theorem~\ref{thm:DevExpo}) is the following. For any subset of arms $\cS \subseteq \{1,\dots,K\}$, for all $x > 0$,
\begin{equation}\label{eq:GenericDev}
 \bP\left(\exists t \in \N : \sum_{a\in \cS} N_a(t)d(\hat{\mu}_a(t),\mu_a) > \sum_{a\in \cS} c\ln(d+\ln(N_a(t))) + |\cS|\cC\left(\frac{x}{|\cS|}\right)\right) \leq  e^{-x}
\end{equation}
where $c$ and $d$ are two positive constants and $\cC(x)$ is a threshold function.
This result holds \emph{for any subset of arms} $\cS$. Combining \eqref{eq:GenericDev} with a weighted union bound, one obtains in Lemma~\ref{lem:AnySubset} below a deviation inequality that is uniform over subsets belonging to the support of a weight vector $\tilde\pi$.

\begin{lemma}[weighted union bound] \label{lem:AnySubset} Assume \eqref{eq:GenericDev} holds. Let $\tilde\pi$ be a probability distribution over subsets: $\sum_{\cS \subseteq \{1,\dots,K\}} \tilde \pi(\cS) = 1$. Then for all $x > 0$
\[
 \bP\left(\exists t , \exists \cS : \!\sum_{a\in \cS} N_a(t)d(\hat{\mu}_a(t),\mu_a) \!> \!\sum_{a\in \cS} c\ln(d+\ln(N_a(t))) + |\cS|\cC\left(\!\frac{x - \ln (\tilde{\pi}(\cS))}{|\cS|}\right)\!\right) \! \leq  e^{-x}.
\]
\end{lemma}

\begin{proof}
  A union bound followed by inequality \eqref{eq:GenericDev} gives
  \begin{align*}
    &\bP\left(\exists t , \exists \cS : \!\sum_{a\in \cS} N_a(t)d(\hat{\mu}_a(t),\mu_a) \!> \!\sum_{a\in \cS} c\ln(d+\ln(N_a(t))) + |\cS|\cC\left(\!\frac{x - \ln (\tilde{\pi}(\cS))}{|\cS|}\right)\!\right)
    \\
    &~\le~
      \sum_{\cS \subseteq \{1,\dots,K\}} \bP\left(\exists t : \!\sum_{a\in \cS} N_a(t)d(\hat{\mu}_a(t),\mu_a) \!> \!\sum_{a\in \cS} c\ln(d+\ln(N_a(t))) + |\cS|\cC\left(\!\frac{x - \ln (\tilde{\pi}(\cS))}{|\cS|}\right)\!\right)
    \\
    &~\le~
      \sum_{\cS \subseteq \{1,\dots,K\}}
      e^{-(x - \ln (\tilde{\pi}(\cS)))}
      ~=~
      e^{-x} \sum_{\cS \subseteq \{1,\dots,K\}} \tilde{\pi}(\cS)
      ~=~
      e^{-x}.
  \end{align*}
\end{proof}

We now explain how this result can serve to tighten the analysis of the GLR stopping rule for some particular sequential testing problems, to allow for the use of smaller threshold functions. We later discuss in Section~\ref{sec:ProjectedCI} the impact of this result on the design of confidence regions.

\subsection{Improved Stopping Rules for Best Arm Identification}\label{sec:improved.stopping}

The (fixed-confidence) Best Arm Identification problem is a particular sequential identification problem as defined in Section~\ref{sec:generic.solved} with $\cO_k = \{ \bm \mu : \mu_k > \max_{j \neq k} \mu_j\}$: the goal is to identify the arm with largest mean.  For this particular problem, the GLR statistic \eqref{eq:GLR} rewrites to 
\begin{equation}\hat{\Lambda}_t = \min_{b \neq \hat{\imath}(t)} \min_{\lambda_b > \lambda_{\hat{\imath}(t)}} \left[N_{\hat{\imath}(t)}(t)d(\hat{\mu}_{\hat{\imath}(t)}(t) , \lambda_{\hat{\imath}(t)}) + N_{b}(t)d(\hat{\mu}_{b}(t) , \lambda_b)\right]\label{GLRBAI}\end{equation}
and the associated stopping rule, $\hat{\Lambda}_t > \hat{c}_t(\delta)$, is referred to as the Chernoff stopping rule by \cite{GK16}.
In this particular case, it is possible to propose a smaller threshold than the universal threshold \eqref{UniversalThreshold} that still ensures a $\delta$-correct rule. Indeed, the probability of error of the strategy that stops when $\hat{\Lambda}_t > \hat{c}_t(\delta)$ and outputs $\ihat(\tau)$ is upper bounded as follows, assuming arm $1$ is the arm with largest mean:
\begin{eqnarray*}
\bP(\text{error}) & \leq&  \bP\left(\exists t \in \N , \exists a \neq 1 :  \min_{\lambda_a > \lambda_{1}} \left[N_{a}(t)d(\hat{\mu}_{a}(t) , \lambda_{a}) + N_{1}(t)d(\hat{\mu}_{1}(t) , \lambda_1)\right] > \hat{c}_t(\delta)\right) \\
 & \leq & \bP\left(\exists t \in \N , \exists a \neq 1:  N_{a}(t)d(\hat{\mu}_{a}(t) , \mu_{a}) + N_{1}(t)d(\hat{\mu}_{1}(t) , \mu_1) > \hat{c}_t(\delta)\right) \\
 & = & \bP\left(\exists t , \exists a \neq 1 : \sum_{j \in \{1,a\}} N_j(t) d(\hat{\mu}_j(t),\mu_j) > \hat{c}_t(\delta)\right).
\end{eqnarray*}
From Theorem~\ref{thm:DevExpo} and a union bound over the $K-1$ subsets $\{1,2\}, \dots,\{1,K\}$ (Lemma~\ref{lem:AnySubset} with a weight vector such that $\tilde{\pi}(\{1,a\}) = 1/(K-1)$ for $a\neq 1$) it holds that
\[\bP\left(\exists t , \exists a \neq 1 : \sum_{j \in \{1,a\}} N_j(t) d(\hat{\mu}_j(t),\mu_j) > 3\sum_{j \in \{1,a\}} \ln(1+\ln(N_j(t))) + 2 \cC_{\text{exp}}\left(\frac{\ln \frac{K-1}{\delta}}{2}\right)\right)\leq \delta.\]
This implies that the GLR rule is $\delta$-correct with the threshold
\begin{equation}\hat{c}_t(\delta) = 6\ln\left( \ln\left(\frac{t}{2}\right) +1\right) + 2 \cC_{\text{exp}}\left(\frac{\ln \frac{K-1}{\delta}}{2}\right).\label{BetterThreshold}\end{equation}
For large $t$, this will be smaller than the original threshold $\hat{c}_t(\delta) = \ln \frac{2t(K-1)}{\delta}$ proposed by \cite{GK16} in the Bernoulli case. It can hence lead to earlier stopping while preserving the optimal sample complexity guarantees, as this threshold still satisfies the assumptions of Theorem~\ref{thm:OptimalTesting}. Note also that this new threshold provides a better motivation for the stylized $\log((\log(t)+1)/\delta)$ threshold that is sometimes used in best arm identification experiments, and for which the empirical error probability is reported to remain below $\delta$.

\begin{remark} The improved threshold \eqref{BetterThreshold} yields a $\delta$ correct stopping rule, however the corresponding confidence interval \eqref{CI:First} does not satisfy $\bP\left(\forall t \in \N : \bm \mu \in \cC_t(\delta)\right)\geq 1 - \delta$. There is no equivalence between the improved $\delta$-correct stopping rule and improved $\delta$-valid confidence regions. We will discuss the implications of Lemma~\ref{lem:AnySubset} for confidence regions in Section~\ref{sec:ProjectedCI}.
\end{remark}

\subsection{Smaller Thresholds for More General Tests}

The reason why we are able to propose a smaller threshold for the BAI problem is that its GLR statistic \eqref{GLRBAI} only features pairs of arms. In more general tests, the structure of the GLR statistic may also be exploited to allow for a smaller threshold that does not depend on the total number of arms $K$ featuring in the universal threshold \eqref{UniversalThreshold} but on a smaller \emph{effective number} of arms.

\begin{definition}\label{def:Rank}
Consider a sequential identification problem specified by a partition $\cO = \bigcup_{i=1}^M \cO_i$. We say this problem has \emph{rank $R$} if for every $i \in \{1,\dots,M\}$ we can write
\[
  \cO\backslash \cO_i
  ~=~
  \bigcup_{q \in [Q]}
  \setc*{\vlambda \in \cI^K}{(\lambda_{k^{i,q}_1}, \ldots, \lambda_{k^{i,q}_R}) \in \mathcal L_{i,q}}
  ,
\]
for a family of arm indices $k_r^{i,q} \in [K]$ and open sets $\mathcal L_{i,q}$ indexed by $r \in [R]$, $q \in [Q]$ and $i \in [M]$. In words, the rank is $R$ if every set $\cO \setminus \cO_i$ is a finite union of sets that are each defined in terms of \emph{only $R$ arms}.
\end{definition}

The BAI problem has rank 2. Indeed, for all $i \in \{1,\dots,K\}$,
\[\cO\backslash \cO_i = \bigcup_{a \neq i}\left\{ \vlambda \in \cI^K \left| (\lambda_i,\lambda_a) \in \{ (x,y) : x < y\} \right.\right\}.\]
In any testing problem that has rank $R$, the GLR statistic may be rewritten
\[
  \hat{\Lambda}_t
  ~=~
  \min_{q \in [Q]}
  \inf_{\substack{\bm\lambda
      \\
      (\lambda_{k^{\ihat(t),q}_1}, \ldots, \lambda_{k^{\ihat(t),q}_R}) \in \mathcal L_{\ihat(t),q}
    }}
  \sum_{r=1}^R N_{k^{\ihat(t),q}_r}(t) d\left(\hat{\mu}_{k^{\ihat(t),q}_r}(t),\lambda_{k^{\ihat(t),q}_r}\right)
  ,
\]
which yields the expression \eqref{GLRBAI} in the BAI case.

\begin{proposition}\label{prop:PACrank}
Fix an identification problem of rank $R$. Then the GLR stopping rule \eqref{eq:GLR} is $\delta$-correct with threshold
\[
  \hat c_t(\delta) ~=~ 3 R \ln (1+\ln(t/R)) + R \cC_{\text{exp}} \del*{\frac{\ln \frac{M-1}{\delta}}{R}}
\]
\end{proposition}

\begin{proof} Fix $\vmu \in \cO$. For each $i \neq i^*$, $\vmu \in \cO \backslash \cO_i$, thus from Definition~\ref{def:Rank} there exists $q_i$ such that $(\mu_{k^{i,q_i}_1}, \ldots, \mu_{k^{i,q_i}_R}) \in \mathcal L_{i,q_i}$. Then
\begin{align*}
  &
    \pr_\vmu \set*{\text{$\tau_\delta < \infty$ and $\ihat(\tau_\delta) \neq i^*$}}
  \\
  &~\le~
    \pr_\vmu \set*{\exists t : \text{$\hat\Lambda_t \ge \hat c_t(\delta)$ and $\ihat(t) \neq i^*$}}
  \\
  &~=~
    \pr_\vmu \set*{\exists t, i \neq i^* : \text{$\hat\Lambda_t \ge \hat c_t(\delta)$ and $\ihat(t) = i$}}
  \\
  &~=~
    \pr_\vmu \set*{\exists t, i \neq i^* : \min_{q \in [Q]}
  \inf_{\substack{\bm\lambda
      \\
      (\lambda_{k^{i,q}_1}, \ldots, \lambda_{k^{i,q}_R}) \in \mathcal L_{i,q}
    }}
  \sum_{r=1}^R N_{k^{i,q}_r}(t) d\left(\hat{\mu}_{k^{i,q}_r}(t),\lambda_{k^{i,q}_r}\right) \ge \hat c_t(\delta)}
  \\
  &~\le~
    \pr_\vmu \set*{\exists t, i \neq i^* :
  \sum_{r=1}^R N_{k^{i,q_i}_r}(t) d\left(\hat{\mu}_{k^{i,q_i}_r}(t),\mu_{k^{i,q_i}_r}\right) \ge  3R \ln(1+\ln(t/R))+  \cC_{\text{exp}} \del*{\frac{\ln \frac{M-1}{\delta}}{R}}}
  \\
  &~\le~
    \pr_\vmu \set*{\exists t, i\neq i^*: 
    \sum_{r=1}^R N_{k^{i,q_i}_r}(t) d\left(\hat{\mu}_{k^{i,q_i}_r}(t),\mu_{k^{i,q_i}_r}\right) \ge 3\sum_{r=1}^R \ln\left(1+\ln N_{k^{i,q_i}_r}(t)\right)+  \cC_{\text{exp}} \del*{\frac{\ln \frac{M-1}{\delta}}{R}}}
  \\
  &~\le~
    \delta
    ,
\end{align*}
where the last inequality follows from Theorem~\ref{thm:DevExpo} and a union bound over $M-1$ subsets (Lemma~\ref{lem:AnySubset} with a weight vector $\tilde{\pi}(\{k_1^{i,q_i},\dots,k_R^{i,q_i}\}) = 1/(M-1)$ for $i\neq i^*$) together with the concavity of $s \mapsto \ln(1+\ln(s))$ that ensures
  \[
    \sum_{r=1}^R \ln\left(1+\ln N_{k^{i,q_i}_r}(t)\right)
    ~\le~
    R \ln (1+\ln (t/R))
    .
  \]
\end{proof}

\subsubsection{A Rank 4 Example} Assume we are given a collection of $K$ pairs of arms and want to find out which pair has the largest difference (which we think of as profit) between first component (which we think of as revenue) and second component (which we think of as cost). More precisely, we consider a $K \times 2$ array of random sources $X_{ij}$ where $i \in [K]$ and $j \in \set{1,2}$. Let $\mu_{ij} = \ex[X_{ij}]$ denote the means. A strategy samples one arm $A_t = (I_t, J_t)$ per round and its goal is to identify the largest profit pair
\[
  i^*(\vmu) ~=~ \arg\max_{i}~ \mu_{i,1} - \mu_{i,2}
  .
\]
It is easy to check that this problem, which we call \emph{Largest Profit Identification}, has rank 4 and the GLR statistic rewrites to \[
  \hat{\Lambda}_t
  ~=~
  \min_{b \neq \ihat}
  \inf_{\substack{\vlambda \in \mathbb R^{\set{b,\ihat} \times \set{1,2}} \\
      \lambda_{b,1} - \lambda_{b,2} > \lambda_{\ihat,1} - \lambda_{\ihat,2}}}
  \sum_{\substack{a \in \set{b,\ihat} \\ j \in \set{1,2}}} N_{a,j}(t) d\left(\hat{\mu}_{a,j}(t),\lambda_{a,j}\right)
  .
\]
By Proposition~\ref{prop:PACrank} the GLR stopping rule \eqref{eq:GLR.stop} is $\delta$-correct with the threshold
\[
  \hat c_t(\delta) ~=~ 12 \ln (1+\ln(t/4)) + 4 \cC_{\text{exp}} \del*{\frac{\ln \frac{K-1}{\delta}}{4}}
  .
\]

\begin{remark}
  For Largest Profit Identification the oracle weights $\w^*(\vmu)$, which are needed for implementing the asymptotically optimal procedure of Section~\ref{sec:asy.opt}, maximise the concave function $T^*(\vmu)^{-1}$. For both Gaussian and Bernoulli (and possibly more) we can write the objective as a Disciplined Convex Program and solve it efficiently with e.g.\ CVX \citep{cvx}.
\end{remark}

\subsubsection{Best Action Identification in a Game Tree} In the bandit literature, a particular structured identification problem that offers a simple model for Monte Carlo Tree Search in games has been recently studied by \cite{Teraoka14MCTS,GKK16,HuangASM17,NIPS17}. The goal is to quickly identify the action at the root of a (maxmin) game tree whose value is the largest by querying noisy samples of the leaves' values of that tree.

Lemma~8 in \cite{NIPS17} provides an expression for the optimal weights in a depth-two tree, that are then computable using disciplined convex optimization tools (e.g.\ CVX). Furthermore, it can be checked that this identification problem is of rank $L+1$, where $L$ is the maximum number of actions of the second player. This is much smaller than the number of leaves, which is $K \cdot L$ in a game tree where the first player has $K$ moves and the second player has $L$ moves. Assuming the weights (which are only numerically computable) satisfy the continuity assumption of Theorem~\ref{thm:OptimalTesting} (or, if not, by \citealt[Theorem~7]{multiple.answers}), the GLR rule with a rank $L+1$ threshold is asymptotically optimal in combination with the Tracking rule. We note that the existing literature does not provide asymptotically optimal algorithms for best action identification in a game tree, even for depth-two trees.

\section{Projected Confidence Intervals}\label{sec:ProjectedCI}

The deviation inequalities presented in this paper can also be used to build tight confidence regions on (functions of) the parameter $\bm\mu \in \cI^K$. We are particularly interested in building $\delta$-\emph{uniformly valid} confidence regions $\cC_t(\delta)$, that satisfy $\bP\left(\forall t \in \N, \bm \mu \in \cC_t(\delta)\right)\geq 1 - \delta$ for every sampling rule.

Lemma~\ref{lem:AnySubset} in combination with our deviation results allows to build such confidence regions. Indeed for any weight vector $\tilde{\pi}$ over subsets, the following confidence interval is $\delta$-uniformly valid (with $c$ and $d$ as given by the lemma):
\begin{equation}\cC_t^{\tilde\pi}(\delta) := \left\{\!\vlambda \!: \!\forall \cS, \!\sum_{a \in \cS}\! N_a(t)d(\hat{\mu}_a(t),\lambda_a\!)\! \leq \!c \!\sum_{a \in \cS} \!\ln(d + \!\ln N_a(t)\!)\! +\! |\cS| \cC_{\text{exp}} \left(\!\frac{\ln(1/(\tilde\pi(\cS)\delta))}{|\cS|}\!\right)\!\right\}.\label{centralRegion}\end{equation}
A natural question is thus which vector $\tilde\pi$ yields the most interesting confidence region. Answering this question would require to compare complicated shapes in $\R^K$ (like we do for $K=2$ in Figure~\ref{fig:stylised}(a) in the Introduction) and the answer would still depend on the \emph{purpose} of those confidence regions.

In this section we investigate their use for computing confidence intervals on derived quantities of the form $f(\vmu)$, where $f : \mathbb R^K \to \mathbb R$ is some fixed function. Knowing that $\vmu \in \cC_t$, we can immediately conclude that $f(\vmu) \in \cI_t(\delta) \df \setc*{f(\vlambda)}{ \vlambda \in \cC_t}$. The interplay of the structure of the function $f$ and the shape of the confidence region $\cC_t$ will jointly determine the tightness of the \emph{projected confidence interval} $\cI_t(\delta)$. The principal challenge is to find, for each $f$ of interest, a statistically tight $\cC_t$ with a computationally tractable way of computing $\cI_t$. In this section we study two classes of examples, linear $f$ and minima/maxima.

\subsection{Linear Functions}\label{sec:projected.linear}
In this section we consider an arbitrary linear function $f(\vmu) = \v^\top \vmu$ where $\v \in \mathbb R^K$. We will derive our results in the Gaussian case because it admits revealing and explicit closed-form expressions. In that case the confidence region \eqref{centralRegion} is $\delta$-uniformly valid for $c=2$ and $d=4$ and $g = g_G$, as licensed by Corollary~\ref{corr:Gaussian}.
The following two confidence intervals on $\v^\top \vmu$ follow from two extreme choices of weight vectors: one supported on all the singleton sets and one supported on the full set.

\begin{proposition}[Box]
  The following is a $\delta$-uniformly valid confidence interval on $\v^\top \vmu$
  \[
    \cI_t(\delta) = \left[\v^\top \hat \vmu(t)
    \pm \sum_{a \in [K]}  \sqrt{2 \del*{
        C^g \del*{
          \ln \frac{K}{\delta}
        }
        + c\ln(d+\ln(N_a(t)))
      } \frac{v_a^2}{N_a(t)}}\right]
    .
  \]
\end{proposition}
\begin{proof} Simple algebra show that $\cI_t(\delta) = \{\v^T \vlambda, \vlambda \in \cC_t^{\tilde\pi}(\delta)\}$ where $\tilde{\pi}$ is uniform on singletons. Indeed, as $\cC_t^{\tilde{\pi}}(\delta)$ is $\delta$-uniformly valid, it holds that for all $t\in\N$ and $a \in [K]$, $\abs*{\hat \mu_a(t) - \mu_a} \le \sqrt{\frac{2}{N_a(t)} \del*{
        C^g \del*{
          \ln \frac{K}{\delta}
        }
        + c\ln(d+\ln(N_a(t)))
      }}$.
\end{proof}

\begin{proposition}[Ellipse]
  The following is a $\delta$-uniformly valid confidence interval on $\v^\top \vmu$
  \[
    \cI_t(\delta)=\left[   \v^\top \hat \vmu(t)
    \pm\sqrt{
      2 \del*{K C^g \del*{\frac{\ln \frac{1}{\delta}}{K}} + \sum_{a \in [K]} c\ln(d+\ln(N_a(t)))} \sum_{a \in [K]} \frac{v_a^2}{N_a(t) }
    }\right].
  \]
\end{proposition}

\begin{proof} We show that $\cI_t(\delta) = \{\v^T \vlambda, \vlambda \in \cC_t^{\tilde\pi}(\delta)\}$ where $\tilde{\pi}$ is a point-mass on the whole set: $\tilde{\pi}(\{1,\dots,K\})=1$. Letting $C = \sum_{a=1}^K \ln(1+\ln N_a(t)) + K \cC_{\text{exp}}(\ln(1/\delta)/K)$, computing the upper bound of this confidence interval requires to compute
  \[
    \max_\vlambda~ \v^\top \vlambda
    \qquad
    \text{subject to}
    \qquad
    \sum_{a \in [K]} N_a(t) \frac{(\hat \mu_a(t) - \lambda_a)^2}{2} ~\le~ C
    .
  \]
  Introducing Lagrange multiplier $\rho$, we find that this is equivalent to
  \[
    \min_{\rho \ge 0} \max_\vlambda~
    \v^\top \vlambda
    +
    \rho \del*{C - \sum_{a \in [K]} N_a(t) \frac{(\hat \mu_a(t) - \lambda_a)^2}{2}}
    .
  \]
  Solving for $\vlambda$ by cancelling the derivative results in $\lambda_a
    ~=~
    \hat \mu_a(t)
    + \frac{v_a}{\rho N_a(t) }
  $,
  asking us to solve
  \[
    \min_{\rho \ge 0}~
    \v^\top \hat \vmu(t)
    + \sum_{a \in [K]} \frac{v_a^2}{2 \rho N_a(t) }
    + \rho C
    ~=~
    \v^\top \hat \vmu(t)
    + \sqrt{
      2 C \sum_{a \in [K]} \frac{v_a^2}{N_a(t) }
    }
  \]
  where zero $\rho$ derivative is found at $
    \rho
    ~=~
    \sqrt{C^{-1}
             \sum_{a \in [K]} \frac{v_a^2}{2 N_a(t) }
    }
   $. 
  As $\min_\vlambda \v^\top \vlambda = - \max_\vlambda (-\v)^\top \vlambda$, the lower bound of $\cI_t(\delta)$ also follows.
  \end{proof}

\subsubsection{Comparison}
The major difference between the two above bounds is the appearance of the sum outside vs inside of the square root. To get more intuition, let's compare in the special case $N_a(t) = t/K$ and approximate $C^g(x) \approx x$. Then we need to compare
\[
  \norm{\v}_1
  \sqrt{2 \del*{
      \ln \frac{K}{\delta}
      + c\ln(d+\ln(t/K))
    } \frac{K}{t}}
  \quad
  \text{and}
  \quad
  \norm{\v}_2
  \sqrt{
    2 \del*{\ln \frac{1}{\delta} + K c\ln(d+\ln(t/K))}  \frac{K}{t}
  }
  .
\]
We see that the box bound depends on the one-norm of $\v$, whereas the ellipse bound depends on the two-norm of $\v$, which can be smaller by a factor $\sqrt{K}$ (at the price of a factor $K$ multiplying the $\ln \ln t$ term). In a regime of small $\delta$, the ellipse bound can thus be much better than the box bound. Another case of interest is $N_a(t) = t \frac{\abs{v_a}}{\sum_a \abs{v_a}}$, which result from following the oracle weights $\w^*(\vmu)$. Also here the advantage of ellipse over box can again be as large as a factor $\sqrt{K}$.

\subsection{Minimum}

We now turn our attention to $f(\vmu) = \min_a \mu_a$. Estimating the minimum (or, symmetrically, maximum) mean is a natural task in the multi-armed bandit setting (see \citealt{kaufmann2018sequential}). Unlike in the linear case, here the situation is not symmetric. We will study separately the lower and upper confidence bounds
\[
 \LCB^{\tilde{\pi}}_t(\delta) = \min \left\{ \min_{a} \lambda_a : \bm \lambda \in \cC^{\tilde\pi,-}_t(\delta) \right\}\ \ \text{and} \ \ \  \UCB^{\tilde\pi}_t(\delta) = \max \left\{ \min_{a} \lambda_a : \bm \lambda \in \cC^{\tilde\pi,+}_t(\delta) \right\}\]
for the confidence regions 
\[\cC^{\tilde{\pi},\pm}_t(\delta) = \left\{ \bm\lambda : \forall \cS, \sum_{a\in \cS}\left[N_a(t)d^\pm(\hat{\mu}_a(t),\lambda_a) - 3\ln(1+\ln N_a(t))\right]^+ \leq |\cS|\cC_{\text{exp}}\left(\frac{\ln(\tilde{\pi}(\cS)/\delta)}{\delta}\right)\right\}\]
that are both $\delta$-uniformly valid by Lemma~\ref{lem:AnySubset}. It follows that $\bP\left\{\forall t \in \N : \min_a \mu_a \leq \UCB^{\tilde{\pi}}_t(\delta)\right\}\geq 1-\delta$ and $\bP\left\{\forall t \in \N : \min_a \mu_a \geq \LCB^{\tilde{\pi}}_t(\delta)\right\}\geq 1-\delta$. We investigate in each case the tightest possible confidence bound that can be obtained by optimising the choice of the weight vector $\tilde{\pi}$.

\subsubsection{Lower Confidence Bound} A minimum is low whenever \emph{one} entry is low. This means that the $\vlambda \in \cC_t^{\tilde \pi,-}$ of lowest mean will have all entries equal to $\hat \vmu$ except for one. This in turn means that we do not get any mileage out of combining evidence from multiple arms. Instead, the best $\LCB_t^{\tilde \pi}$ is obtained for the choice $\tilde{\pi}(\{k\})=1/K$ (uniform distribution on singletons). We find the following.

\begin{proposition}
  At time $t$, for each arm $a$, let $\theta_a(t) \le \hat \mu_a(t)$ be the solution to
  \[
    N_a(t) d^-(\hat \mu_a(t), \theta_a(t))
    ~=~
    3\ln(1+\ln(N_a(t)))
    +
    \cC_{\text{exp}} \del*{
      \ln \frac{K}{\delta}
    }
  \]
  (note the left-hand side increases with decreasing $\theta_a(t)$, so the solution can be found by binary search). Then 
  \[\bP\left\{\forall t \in \N, \min_a \mu_a \geq  \min_a \theta_a(t)\right\} \geq 1 - \delta.\]
\end{proposition}

\begin{proof}
With the choice $\tilde{\pi}(\{k\})=1/K$, $\cC_t^{\tilde{\pi},-}(\delta)$ is the set of $\bm\lambda$:
\[
  \forall a \in [K] :
  N_a(t) d^-(\hat \mu_a(t), \lambda_a)
  ~\le~
  3\ln(1+\ln(N_a(t)))
  +
  \cC_{\text{exp}} \del*{
    \ln \frac{K}{\delta}
  }
  .
\]
By definition, $\theta_a(t)$ is the lowest possible value for $\lambda_a$, and hence $\min_a \theta_a(t)$ is the lowest possible value for $\min_a \lambda_a$.
\end{proof}

\subsubsection{Upper Confidence Bound} Above, we found that we do not learn much about the lower bound in the presence of many arms. For the upper confidence bound the story is different. We explain in Proposition~\ref{prop:MinUB} how to compute $\UCB_t^{\tilde\pi}$ for a general weight vector $\tilde{\pi}$. We then show that empirically a weight vector supported on \emph{all subsets} can be helpful.

\begin{proposition}\label{prop:MinUB} Let $\theta(t)$ be the solution in $\theta$ to the equation
  \[
    \max_{\cS \subseteq [K]}~\left[
    \sum_{a \in \cS} \sbr*{
      N_a(t) d^+(\hat \mu_a(t), \theta)
      - 3\ln(1+\ln(N_a(t)))
    }^+
    -
    \card{\cS}
    \cC_{\text{exp}} \del*{
      \frac{
        \ln \frac{1}{\delta \tilde \pi(\cS)}
      }{
        \card{\cS}
      }
    }\right]
    ~=~
    0.
  \]
Then $\bP\left\{\forall t \in \N, \min_a \mu_a \leq  \theta(t)\right\} \geq 1 - \delta$.\end{proposition}

\begin{proof} We prove that $\UCB_t^{\tilde{\pi}}(\delta) = \theta(t)$. Let $\bm \lambda \in \cC_t^{\tilde{\pi},+}(\delta)$. By definition, 
\[
  \max_{\cS \subseteq [K]}~\left[
  \sum_{a \in \cS} \sbr*{
    N_a(t) d^+(\hat \mu_a(t), \lambda_a)
    - 3\ln(1+\ln(N_a(t)))
  }^+
  -
  \card{\cS}
  \cC_{\text{exp}} \del*{
    \frac{
      \ln \frac{1}{\delta \tilde \pi(\cS)}
    }{
      \card{\cS}
    }
  }\right]
  ~\le~ 0
  .
\]
What does this tell us about $\min_{a \in [K]} \lambda_a$? Well, consider a candidate value $\theta \ge \min_a \hat \mu_a(t)$ for the minimum. Among bandit models $\vlambda$ with $\min_a \lambda_a = \theta$, the left-hand side above is minimised at $\lambda_a = \max \set*{\hat \mu_a(t), \theta}$ and the maximal value of $\min_{a \in [K]} \lambda_a$ is the maximal value of $\theta$ such that 
\[
  \max_{\cS \subseteq [K]}~\left[
  \sum_{a \in \cS} \sbr*{
    N_a(t) d^+(\hat \mu_a(t), \theta)
    - 3\ln(1+\ln(N_a(t)))
  }^+
  -
  \card{\cS}
  \cC_{\text{exp}} \del*{
    \frac{
      \ln \frac{1}{\delta \tilde \pi(\cS)}
    }{
      \card{\cS}
    }
  }\right]
  ~\le~ 0
  .
\]
We recover the objective in the statement by noting that the left-hand side is a continuous and non-decreasing function of $\theta$.
\end{proof}

\subsubsection{Practical Choice of Weight Vector} The upper bound for a minimum may benefit from considering many subsets $\cS \subseteq [K]$ in the weighted union bound. The reason is that a smaller subset will have a smaller evidence term (summing fewer terms), but it may also have a smaller threshold. Here we investigate the use of cardinality-based weight vectors of the form
$\tilde \pi(\cS) = \pi(\card{\cS})/\binom{K}{\card{\cS}}$ for some probability vector $\pi$ on $\{1,\dots,K\}$.

First, let's consider the computation of $\theta(t)$ for those weight vectors: we are looking for the zero crossing of an increasing function, which can be found by e.g.\ binary search. It remains to efficiently evaluate the objective for a fixed $\theta(t)$. Here we propose to express the objective as
\[
  \max_{k \in [K]}
  \underbrace{
    \max_{\cS \subseteq [K] : \card{\cS} = k}~
    \sum_{a \in \cS} \sbr*{
      N_a(t) d^+(\hat \mu_a(t), \theta(t))
      - 3\ln(1+\ln(N_a(t)))
    }^+
  }_\text{the best set takes the $k$ largest contributors; implement by sorting once.}
  -
  k
  \cC_{\text{exp}} \del*{
    \frac{
      \ln \frac{\binom{K}{k}}{\delta \pi(k)}
    }{
      k
    }
  }
  .
\]
and observe that the best set of size $k$ takes the $k$ arms of largest contribution, which we can look up after  sorting the arms by their contribution. Hence each evaluation of the objective can be obtained in $O(K \ln K)$ time.

We expect combining evidence across arms to be particularly useful when there are several arms with means close to the minimum. We illustrate this empirically in Figure~\ref{fig:UCBcomparison} for a Bernoulli bandit model with $M$ arms with mean 0.1 and 4 more arms with means $0.2,0.3,0.4,0.5$ (thus $K=M+4$), for different values of $M$. We consider the use of a ``Box'' weight vector that is uniform on the singletons ($\pi(1)=1$), a weight vector supported on the whole set of arms ($\pi(K) =1$) and a weight vector that is uniform over subset sizes ($\pi(k) = 1/K$). For each value of $M$, data is collected using uniform sampling and we set $\delta=10^{-10}$ to focus on the high confidence regime. We see that the uniform weight vector consistently leads to smaller upper confidence bounds when compared to Box, with an increased gap when $M$ increases. We also experimented with a ``Zipf'' distribution for $\pi$ ($\pi(k) \propto 1/k$), which performed almost identically as the uniform vector.

\begin{figure}[h]
  \centering
\includegraphics[width=0.47\textwidth]{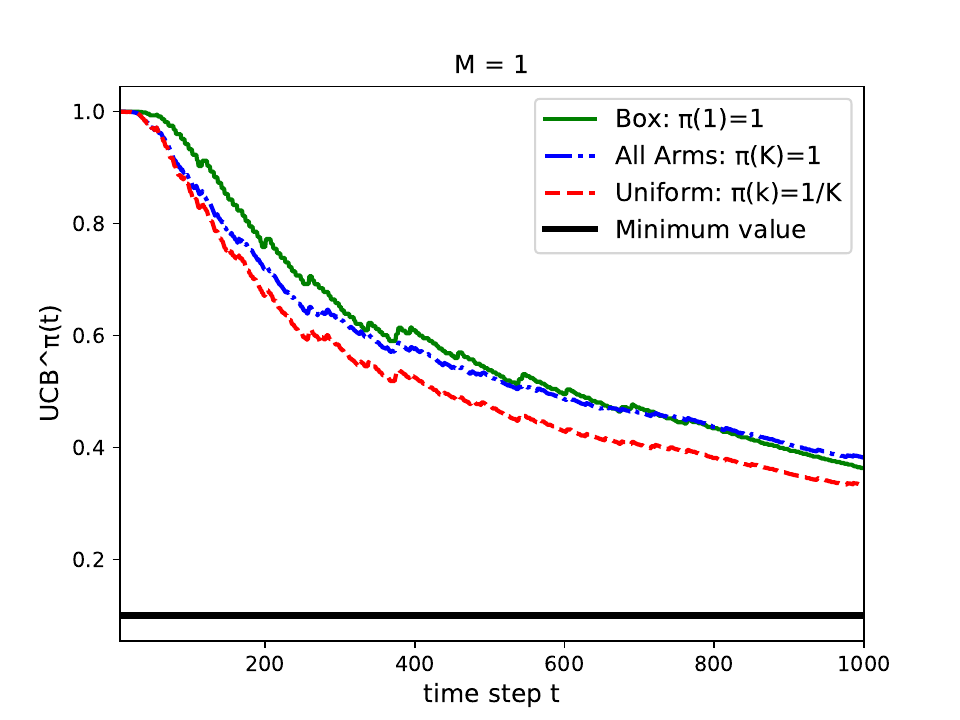}
\includegraphics[width=0.47\textwidth]{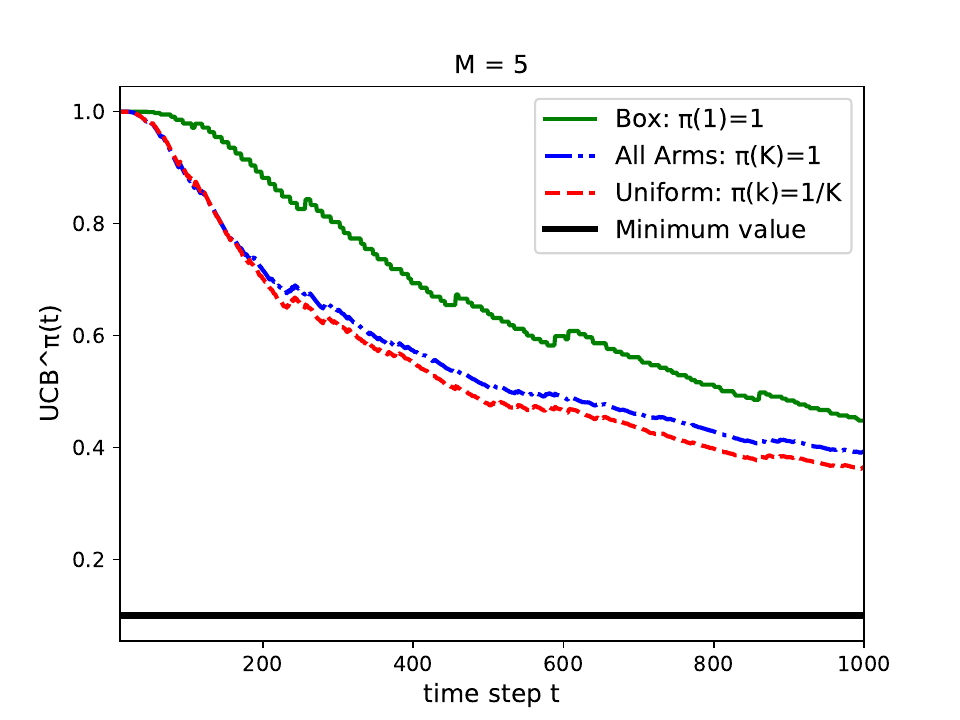}
\includegraphics[width=0.47\textwidth]{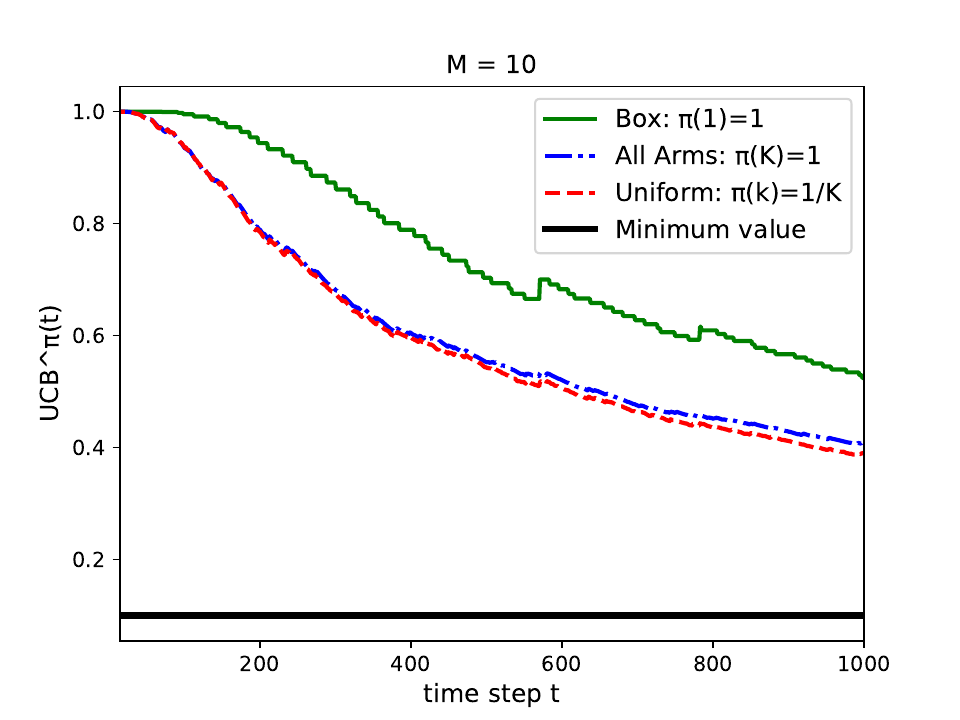}
\includegraphics[width=0.47\textwidth]{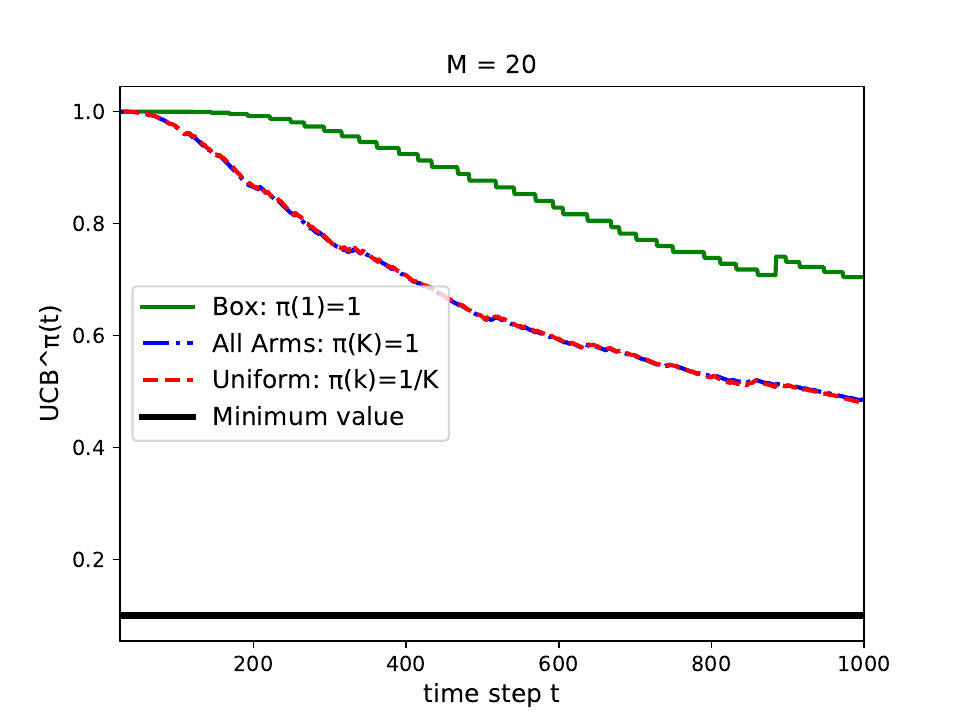}
\caption{\label{fig:UCBcomparison} $\UCB_t^{\tilde{\pi}}(\delta)$ as a function of $t$ for several cardinality-based weight vectors $\tilde{\pi}$ in a presence of $M$ identical arms with the minimal mean, for $M=1,5,10,20$.
}
\end{figure}

This experiment shows that for small values of $\delta$ a uniform cardinality-based weight vector is a robust choice: summing evidence across arms never hurts too much. In the particular case of minimums, we would like to mention that one can go even further and \emph{aggregate} samples from different arms, as explained in \cite{kaufmann2018sequential}, which leads to even smaller upper confidence bounds in experiments.

\section{Conclusion}

Sequential problems are studied in the multi-armed bandit model, where the learner sequentially picks arms to sample. The central question is what the learner infers from the samples that it has seen. This is used for deciding what to do next, when to stop, what to recommend and/or estimate.

We use mixture martingales to design confidence regions, based on self-normalised sums, for exponential family multi-armed bandit models. We argue that these confidence regions are the tightest known, and match, in spirit, established statistical lower bounds.

We then apply the obtained deviation inequalities to the design of
confidence intervals by means of explicit projections,
  stopping rules by means of GLR statistics,
  and asymptotically optimal sampling rules by a tight analysis of the Track-and-Stop algorithm. The fact that we are pushing the state of the art in each of these areas clearly demonstrates the generic appeal of the mixture martingale approach.

\acks{Emilie Kaufmann acknowledges the support of the French Agence Nationale de la Recherche (ANR), under grants BADASS (ANR-16-CE40-0002) and BOLD (ANR-19-CE23-0026-04) and the European CHIST-ERA project DELTA. Wouter Koolen acknowledges support from the Netherlands Organization for Scientific Research (NWO) under Veni grant 639.021.439. We are grateful to INRIA Associate Team ${}^6$PAC.
}

\bibliography{biblioBandits} 

\appendix

\section{Proof of Proposition~\ref{prop:Lambert}} \label{proof:Lambert}
We may write
\[
  h^{-1}(x)
  ~=~
  \inf_{z \ge 1} z \del*{x-1 + \ln \frac{z}{z-1}}
\]
Plugging in the sub-optimal feasible choice $z = 1+\frac{1}{(x-1)+\sqrt{2 (x-1)}}$ reveals
\begin{align*}
  h^{-1}(x)
  &~\le~
   \del*{
    1
    + \frac{1}{(x-1)+\sqrt{2 (x-1)}}
   } \del*{
    x-1 + \ln \del*{x+\sqrt{2 (x-1)}}
    }
  \\
  &~\le~
    1 + (x-1) + \ln \del*{x+\sqrt{2 (x-1)}}
    .
\end{align*}
The last inequality uses
$
  \ln \del*{x+\sqrt{2 (x-1)}}
  \le
  \sqrt{2 (x-1)}
$
which holds with equality at $x=1$ and whose gap is increasing (as can be checked by differentiation).

\section{Additional Proofs for Exponential Families}\label{details:Expo}

\subsection{Proof of Lemma~\ref{lem:ExistsMart}}

Given any probability distribution $\pi$, recall that the associated mixture martingale is defined as 
\[Z_a^{\pi}(t) = \int\exp\left(\lambda S_a(t) - \phi_{\mu_a}(\lambda)N_a(t) \right) \dif \pi(\lambda).\]

The first step of the construction is Lemma~\ref{lem:KLMart}, which relates the deviation of $N_a(t)d^+(\hat{\mu}_a(t),\mu_a)$ and $N_a(t)d^-(\hat{\mu}_a(t),\mu_a)$ to those of $\eta S_a(t) - \phi_{\mu_a}(\eta)N_a(t) $ for a well chosen $\eta$, provided that $N_a(t)$ belongs to some ``slice'' $[(1+\xi)^{i-1},(1+\xi)^i]$.

\begin{lemma}\label{lem:KLMart} Fix $i \in \N^*$, $x>0$ and $\xi > 0$. There exist $\eta_i^+(x,\xi)$ and $\eta_i^-(x,\xi)$ such that, if $N_a(t) \in [(1+\xi)^{i-1},(1+\xi)^i]$ it holds that
\begin{eqnarray*}\left\{N_a(t)d^+(\hat{\mu}_a(t),\mu_a) \geq x\right\} & \subseteq & \left\{ \eta_i^+ S_a(t) - N_a(t) \phi_{\mu_a}(\eta_i^+) \geq \frac{x}{1+\xi}\right\}\\
\left\{N_a(t)d^-(\hat{\mu}_a(t),\mu_a) \geq x\right\} & \subseteq & \left\{ \eta_i^- S_a(t) - N_a(t) \phi_{\mu_a}(\eta_i^-) \geq \frac{x}{1+\xi}\right\}.\end{eqnarray*}
\end{lemma}

The next step is to relate the deviation of $X_a(t)$ to those of a martingale for every $t\in\N$ and not only for $N_a(t)$ is some slice: this will be achieved by a mixture martingale with a well-chosen discrete prior. In the sequel, we consider the (most complicated) case in which $X_a(t) = Y_a(t)$ for all $t$. Given $x$, we define the following probability distribution. Let
\[\begin{array}{cccccc}
   \gamma_i& =& \frac{1}{2}\frac{1}{i^{2} \zeta(2)} & x_i & = & x + \ln\left(\frac{1}{\gamma_i}\right) \\
   \eta_i^+ & = & \eta_i^+(x_i,\xi) & \eta_i^- & = & \eta_i^-(x_i,\xi), \\
  \end{array}
\]
where $\eta_i^{\pm} (x,\xi)$ are defined in Lemma~\ref{lem:KLMart}. We define the discrete prior
\[\pi = \sum_{i =1}^{\infty}\gamma_i \delta_{\eta_i^+} + \sum_{i =1}^{\infty}\gamma_i \delta_{\eta_i^-}\]
and the corresponding mixture martingale 
\[Z_a^{\pi}(t) = \sum_{i=1}^\infty \gamma_i Z_a^{\eta_i^+}(t) + \sum_{i=1}^\infty \gamma_i Z_a^{\eta_i^-}(t),\]
where by a slight abuse of notation, $Z_a^{\eta}(t) = Z_a^{\delta_{\eta}}(t)=\exp(\eta S_a(t) - \phi_{\mu_a}(\eta)N_a(t))$ for $\eta \in \R$. 

In the case $X_a(t) = Y_a^+(t)$, this prior is modified by taking $\gamma_i = \frac{1}{i^{2} \zeta(2)}$ and $ \pi = \sum_{i =1}^{\infty}\gamma_i \delta_{\eta_i^+}$, while for $X_a(t) = Y_a^-(t)$, one defines  $\pi = \sum_{i =1}^{\infty}\gamma_i \delta_{\eta_i^-}.$
We continue the proof assuming $X_a(t) = Y_a(t)$ for all $t$. The proof of the two other cases follow the exact same lines, with the corresponding priors, leading to an improved constant $C(\xi) = \frac{\ln \zeta(2)}{(\ln (1+\xi))^2}$.
\begin{align*}
 & \left\{X_a(t) - (1+\xi) \ln C(\xi) \geq x \right\} \\
 & \ \ \ \subseteq \left\{\left[N_a(t) d(\hat{\mu}_a(t),\mu_a) - 3 \ln(1+\ln(N_a(t)))\right]^+ \geq x + (1+\xi)\ln C(\xi)  \right\} \\
  & \ \ \ = \left\{N_a(t) d(\hat{\mu}_a(t),\mu_a) - 3 \ln(1+\ln(N_a(t))) \geq x + (1+\xi)\ln C(\xi)  \right\},
\end{align*}
where we use that $x + (1+\xi)\ln C(\xi) > 0$ as $\xi < 1/2$. Now, as $2(1+\xi) < 3$, one has
\begin{align*}
 & \left\{X_a(t) - (1+\xi) \ln C(\xi) \geq x \right\} \\
 & \ \ \ \subseteq \left\{N_a(t) d\left(\hat{\mu}_a(t),\mu_a\right)  - 2(1+\xi) \ln\left(1 + \ln(N_a(t))\right)\geq x + (1+\xi)\ln\left(\frac{2\zeta(2)}{\ln(1+\xi)^{2}}\right)\right\} \\
 & \ \ \ \subseteq \left\{N_a(t) d\left(\hat{\mu}_a(t),\mu_a\right) \geq x + (1+\xi)\ln\left(\frac{2\zeta(2) (1+\ln(N_a(t))^{2}}{\ln(1+\xi)^{2}}\right)\right\}\\
 & \ \ \ \subseteq \left\{N_a(t) d\left(\hat{\mu}_a(t),\mu_a\right) \geq x + (1+\xi)\ln\left(\frac{2\zeta(2) (\ln(1+\xi)+\ln(N_a(t))^{2}}{\ln(1+\xi)^{2}}\right)\right\},
\end{align*}
where the last inequality uses $\ln(1+\xi) \leq \ln(3/2) \leq 1$. Now, let $i(t)\geq 1$ be such that $N_a(t) \in [(1+\xi)^{i-1},(1+\xi)^i]$. One can observe that $\frac{\ln N_a(t)}{\ln(1+\xi)} \geq i(t) - 1$. Using Lemma~\ref{lem:KLMart},
\begin{align*}
 & \left\{X_a(t) - (1+\xi) \ln C(\xi) \geq x \right\} \\
 & \ \ \ \subseteq \left\{N_a(t) d\left(\hat{\mu}_a(t),\mu_a\right) \geq x + (1+\xi)\ln\left(\frac{1}{\gamma_{i(t)}}\right)\right\} \\
 & \ \ \ \subseteq \left\{\max_{\eta \in \left\{\eta_{i(t)}^+,\eta_{i(t)}^-\right\}} \left[\eta S_a(t) - \phi_{\mu_a}(\eta)N_a(t)\right]\geq \frac{1}{1+\xi}\left[x + (1+\xi)\ln\left(\frac{1}{\gamma_{i(t)}}\right)\right]\right\}\\
 & \ \ \ \subseteq \left\{\max_{\eta \in \left\{\eta_{i(t)}^+,\eta_{i(t)}^-\right\}} \gamma_{i(t)}\exp\left(\eta S_a(t) - \phi_{\mu_a}(\eta)N_a(t)\right)\geq e^{\frac{x}{1+\xi}} \right\} \\
 & \ \ \ \subseteq \left\{\max_{i\in \N}\max_{\eta \in \left\{\eta_{i}^+,\eta_{i}^-\right\}} \gamma_{i}\exp\left(\eta S_a(t) - \phi_{\mu_a}(\eta)N_a(t)\right)\geq e^{\frac{x}{1+\xi}} \right\} \\
 &\ \ \ \subseteq \left\{Z_a^{\pi}(t)\geq e^{\frac{x}{1+\xi}} \right\}.
\end{align*}

\paragraph{Proof of Lemma~\ref{lem:KLMart}} We introduce the notation $\theta$ for the natural parameter associated to $\mu_a$, defined as $\theta = \dot{b}^{-1}(\mu_a)$. Define $\eta_i^+<0$ and $\eta_i^->0$ such that 
\[\KL(\theta + \eta_i^+, \theta) = \KL(\theta + \eta_i^-, \theta) = \frac{x}{(1+\xi)^{i}}.\]
where $\KL(\theta, \theta')$ is the Kullback-Leibler divergence between the distributions of natural parameter $\theta$ and $\theta'$. Moreover, using some properties of the KL-divergence, one can write
\begin{eqnarray*}
 \KL(\theta + \eta_i^+, \theta) & = & \eta_i^+ \mu_i^+ - \phi_{\mu_a}(\eta_i^+) \ \ \text{with } \ \ \mu_i^+ := \dot{b}^{-1}(\theta + \eta_i^+) < \mu_a, \\
\KL(\theta + \eta_i^-, \theta) & = & \eta_i^- \mu_i^- - \phi_{\mu_a}(\eta_i^-) \ \ \text{with } \ \ \mu_i^- := \dot{b}^{-1}(\theta + \eta_i^-) >\mu_a.
 \end{eqnarray*}
For  $N_a(t) \in [(1+\xi)^{i-1},(1+\xi)^i]$, one has
\begin{eqnarray*}
\left\{N_a(t)d^+(\hat{\mu}_a(t),\mu_a) \geq x\right\}  & \subseteq & \left\{d^+(\hat{\mu}_a(t),\mu_a) \geq \frac{x}{(1+\xi)^i}\right\}\\
 & \subseteq & \left\{\hat{\mu}_a(t) \leq \mu_i^+\right\} \\
 & \subseteq & \left\{\eta_i^+\hat{\mu}_a(t) - \phi_{\mu_a}(\eta_i^+) \geq \KL(\theta + \eta_i^+, \theta) \right\}\\
  & \subseteq & \left\{(1+\xi)^{i-1}\left(\eta_i^+\hat{\mu}_a(t) - \phi_{\mu_a}(\eta_i^+)\right) \geq \frac{x}{1+\xi} \right\}\\
  & \subseteq & \left\{N_a(t) \del*{\eta_i^+\hat{\mu}_a(t) - \phi_{\mu_a}(\eta_i^+)} \geq \frac{x}{1+\xi} \right\},
\end{eqnarray*}
where the third inclusion uses that $\eta_i^+$ is negative. Similarly, using this time that $\eta_i^- > 0$ yields  
\begin{eqnarray*}
\left\{N_a(t)d^-(\hat{\mu}_a(t),\mu_a) \geq x\right\}  & \subseteq &  \left\{\hat{\mu}_a(t) \geq \mu_i^-\right\} \\
 & \subseteq & \left\{\eta_i^-\hat{\mu}_a(t) - \phi_{\mu_a}(\eta_i^-) \geq \KL(\theta + \eta_i^-, \theta) \right\}\\
  & \subseteq & \left\{N_a(t) \del*{\eta_i^-\hat{\mu}_a(t) - \phi_{\mu_a}(\eta_i^-)} \geq \frac{x}{1+\xi} \right\},
\end{eqnarray*}
which concludes the proof.

\subsection{Tight Tuning: Proof of Lemma~\ref{lem:tuning.tight}}\label{sec:tuning}

In this section we prove Lemma~\ref{lem:tuning.tight}, which gives the tightest possible tuning achievable with our method. We first prove two auxiliary lemmas.
\begin{lemma}\label{lem:inner.value}
  Let $x \ge 0$. Then
  \[
    \inf_{q \in [0,1]}~
    \frac{
      x
      - \ln \del*{1-q}
    }{
      q
    }
    ~=~
    h^{-1} \del*{1 + x}
    .
  \]
\end{lemma}
\begin{proof}
  The objective is convex in $\frac{1}{q}$, and hence minimised at zero derivative. Cancelling the derivative requires
\[
  1 + x
  ~=~
  \frac{1}{1-q}
  - \ln \frac{1}{1-q}
  ~=~
  h\del*{\frac{1}{1-q}}
  \qquad
  \text{so that}
  \qquad
  q
  ~=~
  1-\frac{1}{h^{-1} \del*{1 + x}}
\]
where the rewrite in terms of $h$ is allowed since $1/(1-q) \ge 1$. Plugging this in, we find the value as stated.
\end{proof}

\begin{definition}\label{def:htilde}
For any $z \in [1, e]$ and $x \ge 0$, we define
\[
  \tilde h_z(x)
  ~=~
  \min_{y \in [1,z]}~ y \del*{x - \ln \ln y}
  .
\]
\end{definition}

We can now make the connection to \eqref{eq:tilde.h}.

\begin{lemma}\label{lem:outer.value}
Fix $z \in [1, e]$. Then
\[
  \tilde h_z(x)
  ~=~
  \begin{cases}
    \exp \del*{\frac{1}{h^{-1}(x)}} h^{-1}(x)
    & \text{if $x
  \ge
  h \del*{
    \frac{1}{\ln z}
  }
  $,}
\\
z \del*{x - \ln \ln z}
& \text{o.w.}
\end{cases}
\]
\end{lemma}

\begin{proof}
  The objective in Definition~\ref{def:htilde} is convex on $y \in [1,e]$, and its derivative is $x - h(1/\ln y)$. When $x \le h(1/\ln z)$ it is decreasing on the entire domain $y \in [1,z]$, and hence minimised at $y=z$, yielding the second case. If on the other hand $x \ge h(1/\ln z)$, the derivative of the objective is cancelled at
$y = e^{\frac{1}{h^{-1}(x)}}$, and substitution reveals that the value equals
\[
  e^{\frac{1}{h^{-1}(x)}} \del*{x + \ln h^{-1}(x)}
  ~=~
  e^{\frac{1}{h^{-1}(x)}} h^{-1}(x)
  .
\]
\end{proof}

\noindent
We are now ready to prove the Lemma.

\begin{proof}(of Lemma~\ref{lem:tuning.tight})
We reorganise, apply Lemma~\ref{lem:inner.value} and then Lemma~\ref{lem:outer.value} to find
\begin{align*}
  \cC_{\text{exp}}(x)
  &~=~
  \inf_{\xi \in [0,z]}~
  (1+\xi)
  \del*{
    \inf_{q \le 1}~
    \frac{
      x
      - \ln \del*{1-q}
    }{
      q
    }
    + \ln C(\xi)
    }
  \\
  &~=~
  \inf_{\xi \in [0,z]}~
  (1+\xi)
  \del*{
    h^{-1} \del*{1 + x}
    + \ln C(\xi)
    }
  \\
  &~=~
  \inf_{\xi \in [0,z]}~
  (1+\xi)
  \del*{
    h^{-1} \del*{1 + x}
    + \ln \del*{2 \zeta(2)}
    - 2 \ln \ln (1+\xi)
  }
  \\
   &~=~
    2 \inf_{y \in [1,1+z]}~
    y
    \del*{
    \frac{
    h^{-1} \del*{1 + x}
    + \ln \del*{2 \zeta(2)}
    }{2}
    - \ln \ln y
  }
  \\
  &~=~
    2 \tilde h_{1+z}\del*{\frac{h^{-1} \del*{1 + x}
    + \ln \del*{2 \zeta(2)}
    }{2}
    }
    .
\end{align*}
\end{proof}

\subsection{A Tighter One-Arm Bound}\label{proof:onearm}


Lemma~\ref{lem:ExistsMart} allows us to directly derive valid thresholds involving only a single arm. Namely, we have
\begin{corollary}\label{crl:onearm}
  Let $\tilde h_z(x)$ be as defined in \eqref{eq:tilde.h}. For every arm $a$ and confidence parameter $x \ge 0$
  \[
    \pr \set*{\exists t \in \N: 
      X_a(t)
      \ge
      2 \tilde h_{3/2}\del*{
        \frac{
          x + \ln\left(2\zeta(2)\right)
        }{
          2
        }
      }
    }
    ~\le~
    e^{-x}
    .
\]
\end{corollary}
\begin{proof}
  By Lemma~\ref{lem:ExistsMart}, for every $\xi \in [0,1/2]$,
\[
  \pr\left\{X_a(t) - (1+\xi)  \ln\left(\frac{2\zeta(2)}{(\ln (1+\xi))^2}\right) \geq (1+\xi) x\right\} ~\le~ \pr \left\{Z_a^{\pi((1+\xi) x)}(t) \geq e^{x}\right\}
  ~\le~
  e^{-x}
\]
Minimising the threshold w.r.t.\ $\xi$ using Lemma~\ref{lem:outer.value} results in
\[
  \min_{\xi \in [0,1/2]}~ (1+\xi) \del*{ x + \ln\left(\frac{2\zeta(2)}{(\ln (1+\xi))^2}\right)}
  ~=~
  2 \tilde h_{3/2}\del*{
    \frac{
      x + \ln\left(2\zeta(2)\right)
    }{
      2
    }
  }
  .
\]
\end{proof}
We see that the multiple-arm threshold of Theorem~\ref{thm:DevExpo} has $h^{-1}(1+x) > x$ where Corollary~\ref{crl:onearm} has just $x$. This additional blowup is the overhead that our approach incurs for controlling multiple arms by means of a ``Cramér-Chernoff'' approach.

\section{Refined Deviation Inequalities for Gaussian and Gamma Distributions}\label{sec:nicecase}

Theorem~\ref{corr:Gaussian} and Theorem~\ref{corr:Gamma} follow from a similar martingale construction, that could actually be used for other one-dimensional exponential families with divergence function $d(\cdot,\cdot)$, provided one is able to construct a continuous prior satisfying a general assumption given below. 

\begin{assumption}\label{ass:Prior} For every $\lambda \in ]0,1[$, $\mu \in \cI$, there exists a family of functions $(p^{\lambda,\mu}_t)_{t \ge 1}$ such that, for every $t \ge 1$,
\begin{equation}\forall x \in \cI, \ \ \int p^{\lambda,\mu}_t(\eta) e^{\eta t x - \phi_{\mu}(\eta) t} \dif\eta = e^{\lambda t d(x,\mu)}.\label{equ:Prior}\end{equation}
Moreover, for every $1 \le n_1 \le n_2$ and every $\eta \in \R$,
\begin{equation}\label{ineq:Prior}
  p_{n_1}^{\lambda,\mu}(\eta)
  ~\geq~
  \sqrt{\frac{n_1}{n_2}}
  p_{n_2}^{\lambda,\mu}(\eta)
  .
\end{equation}
\end{assumption}

In words, this assumption implies that for all $a$ and $t \ge 1$ there exists a prior distribution for which the corresponding mixture martingale exactly attains $e^{\lambda t d(\hat{\mu}_a(t),\mu_a)}$ and such that one can control the variation of the prior corresponding to two different time steps. Under this assumption, we are able to prove the following.

\begin{theorem}\label{thm:Unified} Assume that Assumption~\ref{ass:Prior} is satisfied and let \[C_0(t,\lambda) := \sup_{\mu \in \cI}\int p_t^{\lambda,\mu} (\eta) \dif\eta.\] Fix $\xi > 0$, $c>1$ and define
\[g_0(\lambda,\xi,c) = \ln \left[\sum_{i=1}^\infty \frac{1}{i^{\lambda c}\zeta(\lambda c)} C_0\left((1+\xi)^{i-1},\lambda\right)\right].\]
The stochastic process $X_a(t) = N_a(t) d(\hat{\mu}_a(t),\mu_a) - c \ln\del[\big]{ \ln(1+\xi) + \ln N_a(t)}$ is $g_{\xi,c}$-VCC where
\[\begin{array}{rcl}
g_{\xi,c} : (c^{-1},1] & \longrightarrow & \R^+ \\
\lambda & \mapsto &  g_0(\lambda,\xi,c) + \frac{1}{2}\ln(1+\xi) + \lambda c \ln \left(\frac{1}{\ln(1+\xi)}\right) + \ln \zeta(\lambda c).
  \end{array}
\]
\end{theorem}

Theorem~\ref{thm:Unified} directly provides a deviation inequality using Lemma~\ref{lem:OneSubset}. It thus remains to find sequences of priors satisfying Assumption~\ref{ass:Prior}, which we were able to do for two particular examples, Gaussian and Gamma distributions. One can note that finding functions $p_t^{\lambda,\mu}$ is closely related to computing a (bilateral) inverse Laplace transform. Indeed, if $q$ is the inverse Laplace transform of $e^{\lambda t d(x, \mu)}$, meaning that $\forall x : \int_{-\infty}^\infty q(s) e^{-s x} \dif s = e^{\lambda t d(x, \mu)}$, the assumption is satisfied for $p_t^{\lambda,\mu}(\eta) = t q(-\eta t) e^{\phi_{\mu}(\eta) t}$. However, computing such inverse Laplace transforms is not easy beyond Gaussian or Gamma distributions.

\paragraph{Proof of Theorem~\ref{thm:Unified}} For $i=1,2,\ldots$ we introduce grid points $T_i = (1+\xi)^{i-1}$ with prior weights $\gamma_i = \frac{1}{i^{\lambda c} \zeta(\lambda c)}$ and define the (un-normalized) martingale
\[\tilde{M}_a^\lambda(t) ~\df~ \sum_{i =1}^\infty \gamma_i \int p_{T_i}^{\lambda, \mu_a}(\eta) e^{\eta S_a(t) - \phi_{\mu_a}(\eta)N_a(t)} \dif\eta,\]
that satisfies $\tilde M_a^\lambda(0) \leq \exp(g_0(\lambda,\xi,c))$.

For $N_a(t) \in [T_i, T_{i+1}[$, we first bound the martingale from below by one of its terms, and then make use of Assumption~\ref{ass:Prior}.
\begin{eqnarray*}
 \tilde{M}_a^\lambda(t) & \geq & \gamma_i \int p_{T_i}^{\lambda, \mu_a}(\eta) e^{\eta S_a(t) - \phi_{\mu_a}(\eta)N_a(t)}\dif\eta \\
 & \geq & \sqrt{\frac{T_i}{N_a(t)}}\gamma_i\int p_{N_a(t)}^{\lambda, \mu_a}(\eta) e^{\eta S_a(t) - \phi_{\mu_a}(\eta)N_a(t)}\dif\eta \\
 &  =  & \sqrt{\frac{T_i}{N_a(t)}}\gamma_i \exp\left(\lambda N_a(t)d\left(\hat{\mu}_a(t),\mu_a\right) \right) \\
 &  \geq   & \sqrt{\frac{1}{1 + \xi}}\gamma_i \exp\left(\lambda N_a(t)d\left(\hat{\mu}_a(t),\mu_a\right) \right),
 \end{eqnarray*}
where the last inequality uses $N_a(t) \leq T_{i+1}$ and $T_i/T_{i+1} = 1/(1+\xi)$, due to the geometric grid.

Introducing the normalised martingale $M_a^\lambda(t) = \tilde M_a^\lambda(t) / \tilde{M}_a^\lambda(0)$ and further using the expression of $\gamma_i$ yields, for all $t$ such that $N_a(t) \in [T_i, T_{i+1}[$,
\begin{eqnarray*}
  {M}_a^\lambda(t) & \geq &
\tilde M_a^\lambda(t)e^{-g_0(\lambda, \xi, c)} ~\geq~
                            e^{\lambda N_a(t)d\left(\hat{\mu}_a(t),\mu_a\right) - g_0(\lambda,\xi,c) - \frac{1}{2}\ln(1+\xi) - \ln \zeta(\lambda c) - \lambda c \ln (i)}.
 \end{eqnarray*}
Finally, using that $i \leq 1 + \ln(N_a(t))/\ln(1+\xi)$ yields the desired
\[
 {M}_a^\lambda(t)
~\geq~
\exp\left(
  \lambda
    X_a(t)
  -
  g_{\xi,c}(\lambda)
\right)
.
\]
It remains to check the case $N_a(t) =0$. Then $X_a(t) = -\infty$, so clearly $M_a^\lambda(t) = 1 > e^{-\lambda \infty}$.
\qed


\subsection{Application to Gaussian Distributions}\label{sec:gaussian.case}

In the Gaussian case, direct computations show that Assumption~\ref{ass:Prior} holds for the choice 
\[p_t^{\lambda,\mu} (\eta) = \frac{1}{\sqrt{1 - \lambda}} \frac{1}{\sqrt{2\pi \sigma_t^2}} \exp\left(-\frac{\eta^2}{2\sigma_t^2}\right),\]
where $\sigma_t^2 = \frac{\lambda}{t(1-\lambda)}$. As a consequence $C_0(t,\lambda) = \frac{1}{\sqrt{1-\lambda}}$ and $g_0(\lambda,\xi,c) = - \frac{1}{2}\ln \left(1 - \lambda\right)$.
Note that the inequality \eqref{ineq:Prior} is actually an equality. We are now ready to prove Theorem~\ref{corr:Gaussian}. 


\paragraph{Proof of Theorem~\ref{corr:Gaussian}} By Theorem~\ref{thm:Unified}, picking $c=2$, for every $\xi>0$ and $\lambda \in ]1/2,1[$ there exists a test martingale $M_a^{\lambda,\xi}(t)$ such that
\[\forall t \in \N, \ M_a^{\lambda,\xi}(t) \geq e^{\lambda\left[N_a(t)d(\hat{\mu}_a(t),\mu_a) - f_{\xi}(N_a(t))\right] - g_{\xi}(\lambda)}\]
with
\begin{eqnarray*}
f_{\xi}(s) & = & 2 \ln( \ln(1+\xi) + \ln(s)) \\
g_{\xi}(\lambda) & = & \frac{1}{2}\ln(1+\xi) + 2\lambda  \ln \left(\frac{1}{\ln(1+\xi)}\right) + \ln \zeta(2\lambda) - \frac{1}{2}\ln \left(1 - \lambda\right)
\end{eqnarray*}
It can be checked that the choice of $\xi$ leading to the smallest $g_{\xi}$ function is $\ln(1+\xi) = 4\lambda $. Denoting by $\xi^*(\lambda)$ this value, it holds that
\begin{eqnarray*}
  g_G(\lambda) ~=~
g_{\xi^*(\lambda)}(\lambda) & = & 2\lambda  - 2\lambda  \ln \left(4\lambda \right) + \ln \zeta(2\lambda) - \frac{1}{2}\ln \left(1 - \lambda\right).
\end{eqnarray*}
For every $\lambda \in ]1/2,1[$, observe that $f_{\xi^*(\lambda)}(s) \leq 2\ln(4+\ln s)$. Hence, there exists a test martingale $M_a^\lambda(t) = M_a^{\lambda,\xi^*(\lambda)}(t)$ such that
\begin{eqnarray*}
  \forall t\in \N, \ M_a^\lambda(t)
& \geq & e^{\lambda\left[N_a(t)d(\hat{\mu}_a(t),\mu_a) - 2 \ln( 4 + \ln(N_a(t)))\right] - g_G(\lambda)},
\end{eqnarray*}
which concludes the proof.

\subsection{Application to Gamma Distributions}\label{sec:gamma.case}

A Gamma distribution with shape parameter $\alpha$ and mean $\mu$ has density at $z > 0$ given by
\[
 f_{\alpha,\mu}(z) = \frac{e^{-\frac{\alpha z}{\mu }} \left(\frac{\alpha z}{\mu }\right)^\alpha}{z \Gamma (\alpha)}
  .
\]
We recover the Exponential distribution for $\alpha=1$. More generally, the set of Gamma distributions with a known shape $\alpha$ form a one-parameter exponential family for which 
\[d(\mu,\mu') = \alpha \left(\frac{\mu}{\mu'} - 1 -\ln \frac{\mu}{\mu'}\right) \ \ \text{ and } \ \ \phi_{\mu}(\eta) = \alpha \ln\left(\frac{\alpha}{\alpha - \mu\eta}\right) \ \text{for } \eta < {\alpha}/{\mu}.\]
Next we show that the family of functions 
\begin{equation}\label{eq:gamma.prior}
  p_t^\lambda(\eta) ~\df~
  \frac{\mu}{\alpha}
  \frac{(\alpha t/e)^{\lambda  \alpha t}}{\Gamma ( \lambda \alpha t)}
  \left(1-\frac{\eta  \mu }{\alpha}\right)^{-\alpha t} \left(\lambda -\frac{\eta  \mu }{\alpha}\right)_+^{\lambda  \alpha t-1}
  .
\end{equation}
leads to suitable ``priors''.
\begin{proposition}
  The family of functions defined in \eqref{eq:gamma.prior} satisfies Assumption~\ref{ass:Prior}.
\end{proposition}
\begin{proof} Proving \eqref{equ:Prior} is equivalent to checking that for all $x > 0$,
\[\frac{\mu}{\alpha}\left(\frac{\alpha t x }{\mu}\right)^{\lambda \alpha t}\frac{1}{\Gamma(\lambda \alpha t)}\int_{-\infty}^{\frac{\lambda \alpha}{\mu}}\left(\lambda - \frac{\eta \mu}{\alpha}\right)^{\lambda\alpha t -1}e^{\eta t x} \dif\eta = e^{\frac{\lambda \alpha t x}{\mu}}\]
which can be done using change of variables to $y = t x \left(\frac{\alpha  \lambda }{\mu }-\eta \right)$ and the definition of the Gamma function $\Gamma(z) = \int_0^\infty x^{z-1} e^{-x} \dif x$.
Now let us check condition \eqref{ineq:Prior}. The condition is trivially satisfied for $\eta \ge \frac{\lambda \alpha}{\mu}$, as both sides are zero. So assume $\eta$ is smaller. Then
\begin{align*}
  \ln
  \frac{
    p_{n_1}^\lambda(\eta)
  }{
    p_{n_2}^\lambda(\eta)
  }
  &~=~
  \ln \frac{
    \frac{\mu}{\alpha}
    \frac{(\alpha n_1/e)^{\lambda  \alpha n_1}}{\Gamma ( \lambda \alpha n_1)}
    \left(1-\frac{\eta  \mu }{\alpha}\right)^{-\alpha n_1} \left(\lambda -\frac{\eta  \mu }{\alpha}\right)^{\lambda  \alpha n_1-1}
  }{
    \frac{\mu}{\alpha}
    \frac{(\alpha n_2/e)^{\lambda  \alpha n_2}}{\Gamma ( \lambda \alpha n_2)}
    \left(1-\frac{\eta  \mu }{\alpha}\right)^{-\alpha n_2} \left(\lambda -\frac{\eta  \mu }{\alpha}\right)^{\lambda  \alpha n_2-1}
    }
  \\
  &~=~
    \ln \frac{\Gamma ( \lambda \alpha n_2) (\alpha n_2/e)^{-\lambda  \alpha n_2}}{\Gamma ( \lambda \alpha n_1) (\alpha n_1/e)^{-\lambda  \alpha n_1}}
    + \alpha (n_2-n_1) \del*{ \ln \left(1-\frac{\eta  \mu }{\alpha}\right)
    - \lambda  \ln \left(\lambda -\frac{\eta  \mu }{\alpha}\right)}
  \\
  &~\ge~
    \frac{1}{2} \ln \left(\frac{n_1 }{n_2}\right)
    + \alpha (n_2 - n_1) \del*{
      \lambda \ln \lambda
    + \ln \left(1-\frac{\eta  \mu }{\alpha}\right)
    - \lambda  \ln \left(\lambda -\frac{\eta  \mu }{\alpha}\right)
    }
  \\
  &~\ge~
    \frac{1}{2} \ln \left(\frac{n_1 }{n_2}\right)
    .
\end{align*}
For the first inequality we used that the approximation error $\ln (\Gamma (x))-x \ln (x)+x-\frac{1}{2} \ln \left(\frac{2 \pi }{x}\right)$ is a decreasing function of $x \in \R_+$ (as can be easily verified by a plot), so that in particular
\[
  \ln \frac{\Gamma (\lambda \alpha n_2)}{\Gamma (\lambda \alpha n_1)}
  ~\ge~
  \frac{1}{2} \ln \left(\frac{n_1 }{n_2}\right)
  + \lambda \alpha n_2 \ln (\lambda \alpha n_2/e)
  - \lambda \alpha n_1 \ln (\lambda \alpha n_1/e)
  .
\]
For the second inequality we use that the expression above switches from decreasing to increasing at $\eta=0$, and is hence minimised there. Plugging in the value $\eta=0$ gives the result.
\end{proof}


We are now ready to prove Theorem~\ref{corr:Gamma}, as a consequence of Theorem~\ref{thm:Unified}. 

\paragraph{Proof of Theorem~\ref{corr:Gamma}} In order to evaluate the function $g_0(\lambda,\xi,c)$ featured in Theorem~\ref{thm:Unified}, we first compute
\[
  C_0(t,\lambda)
  ~=~
  \frac{\Gamma ((1-\lambda) \alpha t)}{\Gamma (\alpha t)}
  (\alpha t/e)^{\lambda  \alpha t} (1-\lambda )^{-(1-\lambda) \alpha t}
  .
\]
To see this, perform the variable substitution $z = \frac{\alpha  \lambda -\eta  \mu }{\alpha -\eta  \mu } \in [0,1]$ to render this a standard Beta integral
\begin{align*}
  C_0(t,\lambda)  &~=~
  \frac{(\alpha t/e)^{\lambda  \alpha t}}{\Gamma ( \lambda \alpha t)}
                    \int_{-\infty}^{\frac{\lambda \alpha}{\mu}}
                    \left(1-\frac{\eta  \mu }{\alpha}\right)^{-\alpha t} \left(\lambda -\frac{\eta  \mu }{\alpha}\right)^{\lambda  \alpha t-1}
                    \frac{\mu}{\alpha}
                    \dif \eta
  \\
                  &~=~
                    \frac{(\alpha t/e)^{\lambda  \alpha t}}{\Gamma ( \lambda \alpha t)}
 \int_0^1
    \left(1-\frac{\lambda-z}{1-z}\right)^{-\alpha t} \left(\lambda -\frac{\lambda-z}{1-z}\right)^{\lambda  \alpha t-1}
    \frac{1-\lambda}{(1-z)^2}
  \dif z
  \\
  &~=~
    \frac{(\alpha t/e)^{\lambda  \alpha t}}{\Gamma ( \lambda \alpha t)}
    \left(1-\lambda\right)^{-(1-\lambda)\alpha t}
    \int_0^1
    z^{\lambda  \alpha t-1}
    \left(1-z\right)^{(1-\lambda) \alpha t -1}
    \dif z
  \\
  &~=~
    (\alpha t/e)^{\lambda  \alpha t}
    \left(1-\lambda\right)^{-(1-\lambda)\alpha t}
    \frac{\Gamma((1-\lambda) \alpha t )}{\Gamma(\alpha t)}
\end{align*}

\begin{proposition}
  $C_0(t,\lambda)$ is decreasing in $t \in \R_+$.
\end{proposition}

\begin{proof}
Let $\psi^{(0)}(x) = \frac{\partial \ln \Gamma(x)}{ \partial x}$. The derivative of $\ln C_0(t,\lambda)$ w.r.t.\ $t$ is negative iff
\[
(1-\lambda) \psi ^{(0)}((1-\lambda) \alpha t)- (1-\lambda)  \ln ((1-\lambda ) \alpha t)
~<~
\psi ^{(0)}(\alpha t) - \ln (\alpha t)
.
\]
Now this follows from the fact that $ x \psi ^{(0)}(x) - x \ln x$ can be checked to be an  increasing function of $x \in \R_+$. 
\end{proof}

We find that $C_0(t,\lambda)$ decreases from $\frac{1}{1-\lambda}$ at $t \to 0$ to $\frac{1}{\sqrt{1-\lambda}}$ for $t \to \infty$. For the former, we use
\begin{align*}
  C_0(t,\lambda)
  &~=~
    \frac{\Gamma ((1-\lambda) \alpha t)}{\Gamma (\alpha t)}
  (\alpha t/e)^{\lambda  \alpha t} (1-\lambda )^{-(1-\lambda) \alpha t}
  \\
  &~=~
  \frac{1}{1-\lambda}
  \frac{((1-\lambda) \alpha t) \Gamma ((1-\lambda) \alpha t)}{(\alpha t) \Gamma (\alpha t)}
    (\alpha t/e)^{\lambda  \alpha t} (1-\lambda )^{-(1-\lambda) \alpha t}
  \\
  &~=~
  \frac{1}{1-\lambda}
  \frac{\Gamma (1+(1-\lambda) \alpha t)}{\Gamma (1+\alpha t)}
  (\alpha t/e)^{\lambda  \alpha t} (1-\lambda )^{-(1-\lambda) \alpha t}
\end{align*}
The claimed limit for $t \to 0$ now follows by taking the limit of each factor, using $\Gamma(1) = 1$ and $t^t \to 1$. For the latter, the first-order Stirling's approximation $\Gamma(z) \sim \sqrt{2 \pi} e^{-z} z^{z-\frac{1}{2}}$ yields
\[
  C_0(t,\lambda) \sim \frac{1}{\sqrt{1-\lambda}} \ \text{ when } \ t \rightarrow \infty
\]
Finally, we have that for all $\lambda \in (0,1)$ and $t\in \N$,
\[C_0(t,\lambda) \in \left[\frac{1}{\sqrt{1-\lambda}} ; \frac{1}{1-\lambda}\right].\]
It follows that for all $\xi > 0$, $-\frac{1}{2}\ln(1-\lambda) \leq g_0(\lambda,\xi,c) \leq -\ln(1-\lambda)$. We might be able to show that $g_0$ is actually closer to $-\frac{1}{2}\ln(1-\lambda)$ as the Stirling approximation is known to be good for moderate values of $t$. However using Theorem~\ref{thm:Unified} (and picking $c=2$) one can already prove that for every $\xi>0$  and $\lambda \in ( c^{-1},1)$, there exists a test martingale $M_a^{\lambda,\xi}(t)$ such that
\[\forall t \in \N, \ M_a^{\lambda,\xi}(t) \geq e^{\lambda\left[N_a(t)d(\hat{\mu}_a(t),\mu_a) - f_{\xi}(N_a(t))\right] - g_{\xi}(\lambda)}\]
with
\begin{eqnarray*}
f_{\xi}(s) & = & 2 \ln( \ln(1+\xi) + \ln(s)) \\
g_{\xi}(\lambda) & = & \frac{1}{2}\ln(1+\xi) + 2\lambda  \ln \left(\frac{1}{\ln(1+\xi)}\right) + \ln \zeta(2\lambda) - \ln \left(1 - \lambda\right).
\end{eqnarray*}
Just like in the proof of Corollary~\ref{corr:Gaussian}, the function $g$ is optimised in $\xi$ at $\ln(1+\xi) = 4 \lambda$. We conclude similarly that $X_a(t) = N_a(t)d(\hat{\mu}_a(t),\mu_a) - 2\ln(4+\ln(N_a(t))$
is $g_\Gamma$-VCC (see Definition~\ref{def:Central}) for the function $g_\Gamma(\lambda) ~=~ 2 \lambda - 2\lambda  \ln \left(4 \lambda\right) + \ln \zeta(2\lambda) - \ln \left(1 - \lambda\right)$.

\qed

\section{Optimal Sample Complexity: Proof of Theorem~\ref{thm:OptimalTesting}}\label{proof:SC}

The first ingredient of the proof is a (deterministic) property of the Tracking sampling rule, that reformulates Lemma 8 in \cite{GK16}. 

\begin{lemma}\label{lem:Tracking} Under the Tracking rule for each $a \in \{1,\dots,K\}$,  $N_a(t) \geq (\sqrt{t} -K/2)_+ -1$. Moreover, for all $\epsilon >0$, for all $t_0$, there exists $t_\epsilon \geq t_0$ such that 
\[\sup_{t\geq t_0}\max_{a \in \{1,\dots,K\}} |w_a^*(\hat{\bm \mu}(t)) - w_a^*(\bm \mu)| \leq \epsilon \ \ \ \Rightarrow \ \ \ \sup_{t \geq t_\epsilon} \max_{a \in \{1,\dots,K\}} \left|\frac{N_a(t)}{t} - w_a^*(\bm \mu)\right| \leq 3 (K-1)\epsilon\;.\]
\end{lemma}

To ease the notation, we fix $\bm \mu\in \cO_1$. From the continuity of $\w^*$ in $\bm\mu \in \cO_1$, there exists $\xi=\xi(\epsilon,\bm\mu)$ such that 
\[\cI_\epsilon := [\mu_1 - \xi,\mu_1 + \xi] \times \dots \times [\mu_K - \xi, \mu_K + \xi]\]
is included in $\cO_1$ and is such that for all $\bm\mu' \in \cI_\epsilon$, \[\max_{a \in \{1,\dots,K\}} |w_a^*(\bm\mu') - w^*_a(\bm\mu)| \leq \epsilon.\]
In particular, whenever $\hat{\bm \mu}(t) \in \cI_\epsilon$, it holds that $\ihat(t) = 1$.

\bigskip\noindent
Let $T \in \N$ and define the ``good tail'' event
\[\cE_T(\epsilon)= \bigcap_{t = T^{1/4}}^{T}\left(\hat{\bm \mu}(t) \in \cI_\epsilon \right).\]
By Lemma~\ref{lem:Tracking}, under the Tracking rule each arm is drawn at least of order $\sqrt{t}$ times at round $t$. This permits to establish the following concentration result, stated as Lemma 19 in \cite{GK16}.  

\begin{lemma}\label{lem:concSimple} There exist two constants $B,C$ (that depend on $\bm \mu$ and $\epsilon$) such that \[\bP_{\bm \mu}(\cE_T^c(\epsilon)) \leq B T \exp(-C T^{1/8}).\] 
\end{lemma}

Using Lemma~\ref{lem:Tracking}, there exists a constant $T_\epsilon$ such that for $T \geq T_\epsilon$, it holds that on $\cE_T(\epsilon)$, 
 \[\forall t \geq \sqrt{T}, \ \max_{a\in \{1,\dots,K\}}\left|\frac{N_a(t)}{t} - w_a^*(\mu)\right| \leq 3(K-1)\epsilon\]
On the event $\cE_T(\epsilon)$, for $t \geq T^{1/4}$ it holds that $\ihat(t)=1$, thus $\Alt(\hat{\bm\mu}(t)) = \Alt(\bm\mu)$ and $\hat{\Lambda}_t = t \hat{M}(t)$ where
\[\hat{M}(t) : =  \inf_{\bm\lambda \in \Alt(\bm\mu)} \sum_{a \in \{1,\dots,K\}} \frac{N_a(t)}{t} d\left(\hat{\mu}_a(t),\lambda_a\right).\]
One can rewrite
\begin{eqnarray*}
\hat{M}(t) & = & g\left(\hat{\bm\mu}(t), \left(\frac{N_a(t)}{t}\right)_{a \in \{1,\dots,K\}}\right), 
\end{eqnarray*}
with $g$ a mapping defined on $\cO_1\times [0,1]^{K}$ by 
\[g(\bm\mu',\bm w') = \inf_{\bm\lambda \in \Alt(\bm\mu)} \sum_{a \in \{1,\dots,K\}} w'_a d\left(\mu'_a,\lambda_a\right).\]
As the mapping $(\bm \lambda, \bm \mu',\bm w') \mapsto \sum_{a \in \{1,\dots,K\}} w'_a d\left(\mu'_a,\lambda_a\right)$ is jointly continuous and the constraint set $\Alt(\bm\mu)$ doesn't depend on $(\bm \mu',\bm w')$, it follows from the application of Berge's maximum theorem \citep{Berge63} that $g$ is continuous.

For $T\geq T_\epsilon$, introducing the constant 
\[C^*_\epsilon(\bm \mu) = \inf_{\substack{\bm \mu' : ||\bm\mu' - \bm \mu|| \leq \xi(\epsilon) \\ \bm w' : ||\bm w' -\w^*(\bm \mu)||\leq 3(K-1)\epsilon}} g(\bm \mu',\bm w')\;,\]
on the event $\cE_T(\epsilon)$ it holds that for every $t \geq \sqrt{T}$, $\hat{M}(t) \geq C^*_\epsilon(\bm \mu)$. 

Let $T \geq T_\epsilon$. On $\cE_T(\epsilon)$, 
\begin{eqnarray*}
 \min\left(\tau_\delta^{\text{GLR}},T\right) & \leq & \sqrt{T} + \sum_{t=\sqrt{T}}^T \ind_{\left(\tau_\delta > t\right)} \leq \sqrt{T} + \sum_{t=\sqrt{T}}^T \ind_{\left(t \hat{M}(t) \leq c_t(\delta)\right)} \\
 & \leq & \sqrt{T} + \sum_{t=\sqrt{T}}^T \ind_{\left(t C_\epsilon^*(\bm\mu) \leq c_T(\delta)\right)} \leq  \sqrt{T} + \frac{c_T(\delta)}{C_\epsilon^*(\bm \mu)}\;.
\end{eqnarray*}
Introducing 
\[T_0^\epsilon(\delta) = \inf \left\{ T \in \N : \sqrt{T} + \frac{c_T(\delta)}{C_\epsilon^*(\bm \mu)} \leq T \right\},\]
for every $T \geq \max (T_0^\epsilon(\delta), T_\epsilon)$, one has $\cE_T(\epsilon) \subseteq (\tau_\delta \leq T)$, therefore 
\[\bP_{\bm \mu}\left(\tau_\delta > T\right) \leq \bP(\cE_T^c) \leq BT \exp(-C T^{1/8})\]
and
\[\bE_{\bm \mu}[\tau_\delta] \leq T_0^\epsilon(\delta) + T_\epsilon + \sum_{T=1}^\infty BT \exp(-C T^{1/8})\;.\]
We now provide an upper bound on $T_0^\epsilon(\delta)$. For $\xi >0$ we introduce the constant
\[C(\xi) = \inf \{ T \in \N : T - \sqrt{T} \geq T/(1+\xi)\}.\]
Using moreover the upper bound on the threshold yields
\[ T_0^\epsilon(\delta)  \leq  C + C(\xi) + \inf \left\{T \in \N : \frac{\ln\left(\frac{DT}{\delta}\right)}{C_\epsilon^*(\bm\mu)} \leq \frac{T}{1+\xi}\right\}.\]
Letting  $h^{-1}$ be the function defined in the statement of Theorem~\ref{thm:DevExpo} which is related to the Lambert function. One has
\begin{eqnarray*}
T_0(\delta)  & \leq & C + C(\xi) +\frac{(1+\xi)}{C_\epsilon^*(\bm \mu)}h^{-1}\left(\ln\left(\frac{(1+\xi)D }{C^*_\epsilon(\bm \mu)\delta}\right)\right).
\end{eqnarray*}
Using Proposition~\ref{prop:Lambert}, it follows that 
\begin{eqnarray*}T_0(\delta) \leq C + C(\xi)  &+& \frac{(1+\xi)}{C_\epsilon(\bm \mu)}\left[\ln\left(\frac{(1+\xi)D}{C^*_\epsilon(\bm \mu)\delta}\right)
 + \ln\left(\ln\left(\frac{(1+\xi)D}{C^*_\epsilon(\bm \mu)\delta}\right)+\sqrt{2\ln\left(\frac{(1+\xi)D}{C^*_\epsilon(\bm \mu)\delta}\right) - 2}\right) \right].\end{eqnarray*}
From this last upper bound, for every $\xi>0$  and $\epsilon>0$,
\[\limsup_{\delta \rightarrow 0} \frac{\bE_{\bm \mu}\left[\tau_\delta^{\text{GLR}}\right]}{\ln(1/\delta)} \leq \frac{(1+\xi)}{C_\epsilon^*(\bm \mu)}.\]
Letting $\xi$ and $\epsilon$ go to zero and using that, by continuity of $g$ and by definition of $\w^*(\bm\mu)$,
\[\lim_{\epsilon \rightarrow 0} C_\epsilon^*(\bm \mu) = T^*(\bm \mu)^{-1}\]
yields 
\[\limsup_{\delta \rightarrow 0} \frac{\bE_{\bm \mu}[\tau_\delta]}{\ln(1/\delta)} \leq T^*(\bm \mu)\]
To conclude, the lower bound of Proposition~\ref{prop:LB} implies that this inequality is an equality.

\end{document}